\documentclass[10pt,twocolumn,letterpaper]{article}

\usepackage{wacv}
\usepackage{times}
\usepackage{epsfig}
\usepackage{graphicx}
\usepackage{amsmath}
\usepackage{amssymb}
\usepackage[american]{babel}
\usepackage{balance}
\usepackage[numbers]{natbib}
\usepackage[utf8]{inputenc} 
\usepackage[T1]{fontenc}    
\usepackage{url}            
\usepackage{booktabs}       
\usepackage{amsfonts}       
\usepackage{nicefrac}       
\usepackage{microtype}      

\usepackage{times}
\usepackage{epsfig}
\usepackage{graphicx}
\usepackage{amsmath}
\usepackage{amssymb}
\usepackage{amsthm}  
\usepackage{wasysym}

\usepackage{color}
\usepackage{dsfont}	
\usepackage{wrapfig}
\usepackage{footnote}
\usepackage{epstopdf}
\usepackage{enumitem}
\usepackage{subfigure}
\usepackage{csquotes}
\usepackage{siunitx}
\usepackage{multirow}

\def\ie{\emph{i.e. }}

\def\etal{\emph{et al. }}

\DeclareMathOperator{\sgn}{sgn}

\makeatletter
\newcommand*{\rom}[1]{\expandafter\@slowromancap\romannumeral #1@}
\makeatother
\makeatletter
\newcommand\footnoteref[1]{\protected@xdef\@thefnmark{\ref{#1}}\@footnotemark}
\makeatother

\newtheorem{lemma}{Lemma}
\newtheorem{defi}{Definition}

\usepackage[pagebackref=true,breaklinks=true,letterpaper=true,colorlinks,bookmarks=false]{hyperref}

\wacvfinalcopy 

\def\wacvPaperID{798} 

\ifwacvfinal\fi
\begin{document}

\title{Self-Orthogonality Module: \\A Network Architecture Plug-in for Learning Orthogonal Filters}

\author{Ziming Zhang\thanks{Joint first authors for the paper.} \,\thanks{This work was done when the author was a researcher at Mitsubishi Electric Research Laboratories (MERL).} \\
Worcester Polytechnic Institute, MA\\
{\tt\small zzhang15@wpi.edu}
\and
Wenchi Ma\footnotemark[1] \hspace{1cm} Yuanwei Wu \hspace{1cm} Guanghui Wang \thanks{The work was supported in part by USDA NIFA (2019-
	67021-28996).}\\
University of Kansas, KS\\
{\tt\small \{wenchima, y262w558, ghwang\}@ku.edu}
}

\maketitle
\ifwacvfinal\thispagestyle{empty}\fi

\begin{abstract}

In this paper, we investigate the empirical impact of orthogonality regularization (OR) in deep learning, either solo or collaboratively. Recent works on OR showed some promising results on the accuracy. In our ablation study, however, we do not observe such significant improvement from existing OR techniques compared with the conventional training based on weight decay, dropout, and batch normalization. To identify the real gain from OR, inspired by the locality sensitive hashing (LSH) in angle estimation, we propose to introduce an implicit {\em self-regularization} into OR to push the mean and variance of filter angles in a network towards $\ang{90}$ and $\ang{0}$ simultaneously to achieve (near) orthogonality among the filters, without using any other explicit regularization. Our regularization can be implemented as an architectural plug-in and integrated with an arbitrary network. We reveal that OR helps {\em stabilize} the training process and leads to {\em faster convergence} and {\em better generalization}.
\end{abstract}

\section{Introduction}

Nowadays deep learning has achieved the state-of-the-art performance in computer vision and natural language processing~\cite{cen2019dictionary,he2018learning, xu2019toward, xu2019adversarially, wu2019unsupervised}. Regularization in deep learning plays an important role in helping avoid bad solutions. Researchers have made great effort on this topic from different perspectives, such as data processing~\cite{ cen2019boosting,ioffe2015batch,ba2016layer,cogswell2015reducing}, network architectures~\cite{zeiler2013stochastic,he2016deep,ma2019mdfn,ma2018mdcn, wu2019unsupervised}, losses~\cite{milletari2016v}, regularizers~\cite{rodriguez2016regularizing,NIPS2017_6769,Bansal2018nips,Martin2018,xu2019towards}, and optimization~\cite{bottou1998online,hinton2012improving, NIPS2016_6114,huang2017orthogonal,Zhang_2018_CVPR}. Please refer to~\cite{kukavcka2017regularization} for a review.

To better understand the effects of regularization in deep learning, our work in this paper is mainly motivated by the following two basic yet important questions:
\begin{enumerate}[leftmargin=7mm]
    \item[Q1.]{\em With the help of regularization, what structural properties among the learned filters in hidden layers\footnote{For simplicity, in the rest of the paper we refer to a convolutional or FC layer as a hidden layer.} are good deep models supposed to have?}
\end{enumerate}

To answer this question (partially), we try to explore the angular properties among learned filters. We compute the angles of all filter pairs at each hidden layer in different deep models and plot these {\em angular distributions} in Fig.~\ref{fig:angle_dist}. To generate each distribution, we first uniformly and randomly draw a sample from the angle pool per hidden layer, and average all the samples to generate a model-level angular sample. We then repeat this procedure for $10^6$ times, leading to $10^6$ samples based on which we compute a (normalized) histogram as the angular distribution by quantization from $\ang{0}$ to $\ang{180}$, step by $\ang{0.1}$. All the 23 deep models~\cite{pretrained_model_link}
are properly pretrained on different data sets with weight decay~\cite{goodfellow2016deep}, dropout~\cite{hinton2012improving}, and batch normalization (BN)~\cite{ioffe2015batch}. 

\begin{figure*}[th]
	\begin{center}
		\includegraphics[width=.95\linewidth]{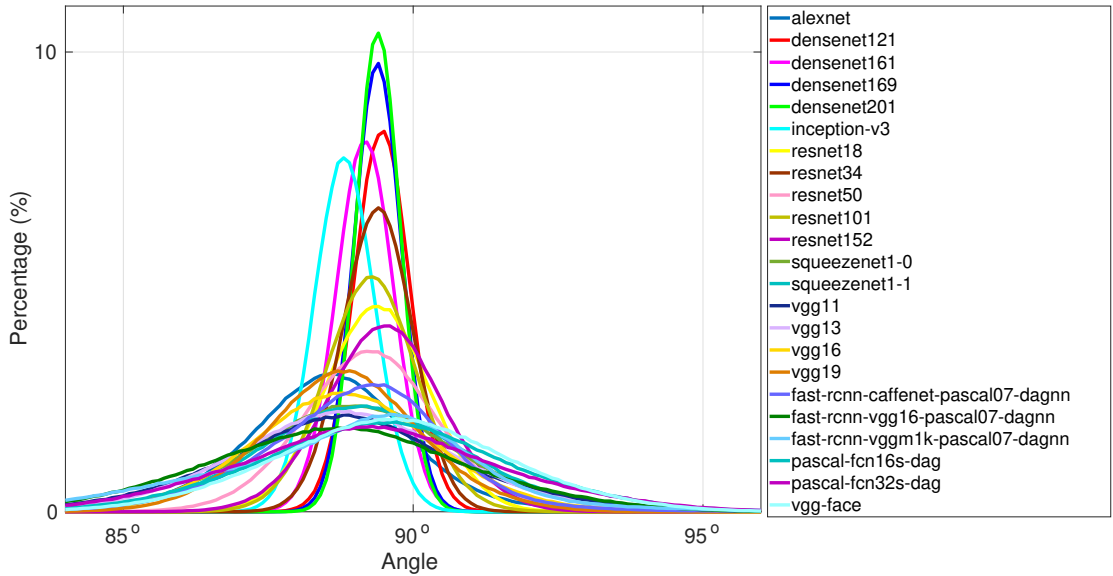}
		\caption{Illustration of the angular distributions of pretrained models with no OR.}
		\label{fig:angle_dist}
	\end{center}
\end{figure*}

As shown in Fig.~\ref{fig:angle_dist}, all the angular distributions overlap with each other heavily and behave similarly in Gaussian-like shapes with centers near $\ang{90}$ with small variances. 
Intuitively orthogonal filters are expected to best span the parameter space, especially in the high dimensional spaces where the filter dimensions are larger than the number of filters. Empirically, however, with many noisy factors such as data samples and stochastic training it may not be a good idea to strictly preserve the filter orthogonality in deep learning. In fact, the recent work in~\cite{2018arXiv181002786K} has demonstrated that on benchmark data sets, classification accuracy using orthogonal filters (learned by PCA) is inferior to that using learned filters by backpropagation (BP). Similarly another recent work in~\cite{vorontsov2017orthogonality} finds that hard constraints on orthogonality can negatively affect the convergence speed and model performance in training of recurrent neural networks (RNNs), but soft orthogonality can improve the training.

In summary, the comparison on the angular distributions of pretrained deep models in Fig.~\ref{fig:angle_dist} reveal that deep learning itself may have some internal mechanism to learn nearly orthogonal filters due to its high dimensional parameter spaces, even without any external orthogonal regularization (OR). 

\begin{enumerate}[leftmargin=7mm]
	\item[Q2.]{\em What are the intrinsic benefits from learning orthogonal filters in deep learning based on OR?}
\end{enumerate}

We notice that recently OR has been attracting more and more attention~\cite{rodriguez2016regularizing, zhang2016hybrid, vorontsov2017orthogonality, huang2017orthogonal, Bansal2018nips, 2018arXiv181002786K}, some of which~\cite{rodriguez2016regularizing, vorontsov2017orthogonality, huang2017orthogonal} have released their code. Interestingly, from their code we find that the proposed OR is evaluated together with other regularizers such as weight decay, dropout, and BN. We argue that such experimental settings cannot help identify how much OR contributes to the performance, compared with other regularizers, especially as we observe that the performances with or without OR are very close. Similar argument has been addressed in~\cite{van2017l2} recently where the author showed that $\ell_2$ regularization has no regularizing effect when combined with batch or weight normalization, but has an influence on the scale of weights, and thereby on the effective learning rate. 

In summary, it is unclear to us from existing works what is the real gain from OR in deep learning.

{\bf Contributions:} 
This paper aims to identify the real gain from OR in training different deep models on different tasks. To do so, we conduct comprehensive experiments on point cloud classification. In contrast to previous works, we separate OR from other regularization techniques to train the same networks respectively. We observe that, however, no significant improvement in accuracy occurs from existing OR techniques, statistically speaking, compared with the conventional training algorithm based on weight decay, dropout, and batch normalization. In fact, we find that, even without any regularization, a workable deep model can achieve the near orthogonality among learned filters, indicating that OR may not be necessarily useful in deep learning to improve accuracy. 

What we do observe is that sometimes the training stability using OR is improved, leading to faster convergence in training and better accuracy at test time. We manage to identify this by intentionally designing experiments in extreme learning scenarios such as large learning rate, limited training samples, and small batch sizes. Such observation, however, is not strong overall. We conjecture that this is mainly because existing OR techniques influence the deep learning {\em externally} and cannot be integrated as a part of network architectures {\em internally}. 

To verify our conjecture, we propose a {\em self-regularization} technique as a plug-in to the network architectures so that they are able to learn (nearly) orthogonal filters even without any other regularization. We borrow the idea from locality sensitive hashing (LSH)~\cite{charikar2002similarity} to approximately measure the filter angles at each hidden layer using filter responses from the network. We then push the statistics of such angles (\ie mean and variance) towards $\ang{90}$ and $\ang{0}$, respectively, as an orthogonality regularizer. We demonstrate that our internal self-regularization significantly improves the training stability, leading to faster convergence and better generalization.

\subsection{Related Work}
As summarized in~\cite{kukavcka2017regularization}, there are many regularization techniques in deep learning. For instance, weight decay is essentially an $\ell_2$ regularizer over filters, dropout takes random neurons for update, and BN utilizes the statistics from mini-batches to normalize the features. Our work is more related to representation decorrelation and orthogonality regularizers in the literature. 

{\bf Representation Decorrelation:} Cogswell~\etal~\cite{cogswell2015reducing} proposed a regularizer, namely DeCov, to learn non-redundant representations by minimizing the cross-covariance of hidden activations. Similarly, Gu~\etal~\cite{ijcai2018-301} proposed another regularizer, namely Ensemble-based Decorrelation Method (EDM), by minimizing the covariance between all base learners (\ie hidden activations) during training. Yadav and Agarwal~\cite{yadav2017aaaiw} proposed to regularize the training of RNNs by minimizing non-diagonal elements of the correlation matrix computed over the hidden representation, leading to DeCov RNN loss and DeCov Ensemble loss. Zhu~\etal~\cite{ijcai2018-453} proposed another decorrelation regularizer based on Pearson correlation coefficient matrix working together with group LASSO to learn sparse neural networks. However, none of the previous work can guarantee that the learned filters are (nearly) orthogonal. Different from these approaches working on hidden activation (representations) by encouraging diverse or non-redundant representations, or like dropout which directly works on neurons by random screening, the proposed method essentially works on regularizing the filter parameters, specifically the weights, to update filters towards orthogonality. It is data adaptive by extracting weights' activation for the regularizer so that the update during training is data dependent. 


{\bf Orthogonality Regularizers:} Harandi and Fernando \cite{harandi2016generalized} proposed a generalized BP algorithm to update filters on the Riemannian manifolds as well as introducing a Stiefel layer to learn orthogonal filters. Vorontsov~\etal~\cite{vorontsov2017orthogonality} verified the effect of learning orthogonal filters on RNN training that is conducted on the Stiefel manifolds. Huang~\etal~\cite{huang2017orthogonal} proposed an orthogonal weight normalization algorithm based on optimization over multiple dependent Stiefel manifolds (OMDSM). Xie~\etal~\cite{xie2017near} proposed a family of orthogonality-promoting regularizer by encouraging the Gram matrix of the functions in the reproducing kernel Hilbert spaces (RKHS) to be close to an identity matrix where the closeness is measured by Bregman matrix divergences. Rodr{\'\i}guez~\etal~\cite{rodriguez2016regularizing} proposed a regularizer called OrthoReg to enforce feature orthogonality locally based on cosine similarities of filters. Bansal~\etal~\cite{Bansal2018nips} proposed another two orthogonality regularizers based on mutual coherence and restricted isometry property over filters, respectively, and evaluated their gain in training deep models. Xie~\etal~\cite{xie2017all} demonstrated that orthonormality among filters helps alleviate the vanishing or exploring gradient issue in training extremely deep networks. Jia~\etal~\cite{jia2019orthogonal} proposed the algorithms of Orthogonal Deep Neural Networks (OrthDNNs) to connect with recent interest of spectrally regularized deep learning methods.

{\bf Self-Regularization:}
Xu~\etal~\cite{xu2018srnn} proposed a self-regularized neural networks (SRNN) by arguing that the sample-wise soft targets of a neural network may have potentials to drag its own neural network out of its local optimum. Martin \& Mahoney~\cite{Martin2018} proposed interpreting deep neural networks (DNN) from the perspective of random matrix theory (RMT) by analyzing the weight matrices in DNN. They claimed that empirical and theoretical results clearly indicate that the DNN training process itself implicitly implements a form of self-regularization, implicitly sculpting a more regularized energy or penalty landscape.

Differently, we propose introducing self-regularization into the design of OR for better understanding the truly impact of OR in deep learning. In contrast to~\cite{Martin2018} we discover the self-regularization in convolutional neural networks (CNNs) by considering their angular distributions among learned filters in the context of OR. We propose a novel efficient activation function to compute these angles.

\section{Self-Regularization as Internal Orthogonality Regularizer}

To better present our experimental results, let us first introduce our self-regularization method. Recall that inspired by the angular distributions of pretrained deep models, we aim to learn deep models with mean and variance of their angular distributions close to $\ang{90}$ and $\ang{0}$ as well. Intuitively we could use $\theta=\arccos{\frac{\mathbf{w}_n^T\mathbf{w}_n}{\|\mathbf{w}_m\|\|\mathbf{w}_n\|}}\in[0, \pi]$ to directly compute the angle between two filters $\mathbf{w}_m$ and $\mathbf{w}_n$, where $(\cdot)^T$ denotes the matrix transpose operator, and $\|\cdot\|$ denotes the $\ell_2$ norm of a vector. Empirically we observe similar behavior of this $\arccos$ based regularization to that of SRIP-v1 and SRIP-v2 (see Sec.~\ref{sec:exp} for more details). We argue that all the existing orthogonality regularizers are designed {\em data independently}, and thus lack of the ability of data adaptation in regularization.  

In contrast to the literature, we propose introducing implicit {\em self-regularization} into OR to embed the regularizer into the network architectures directly so that it can be updated {\em data dependently} during training. To this end, we actively seek for a means that can be used to estimate filter angles based on input data. Only in this way we can naturally incorporate self-regularization with network architectures. Such requirement reminds us of the connection between LSH and angle estimation, leading to the following claim:
\begin{lemma}[$\vartheta$-Space]\label{lem:1}
	Without loss of generality, let $\theta\in[0, \pi]$ be the angle between two vectors $\mathbf{w}_m, \mathbf{w}_n\in\mathbb{R}^d$, and $\mathcal{X}$ be the unit ball in the $d$-dimensional space. We then have
	\begin{align}\label{eqn:g}
		\vartheta & \stackrel{def}{=} \mathbb{E}_{\mathbf{x}\sim\mathcal{X}}\left[\sgn(\mathbf{x}^T\mathbf{w}_m)\sgn(\mathbf{x}^T\mathbf{w}_n)\right] = 1 - \frac{2\theta}{\pi},
	\end{align}
	where $\mathbb{E}$ is the expectation operator, sample $\mathbf{x}$ is uniformly sampled from $\mathcal{X}$, $\sgn:\mathbb{R}\rightarrow\{\pm1\}$ is the sign function returning 1 for positives, otherwise -1.
\end{lemma}
\begin{proof}
	In fact $\sgn(\mathbf{x}^T\mathbf{w})$ defines a random hyperplane based hash function for LSH \cite{charikar2002similarity}. Therefore, 
	\begin{align}\label{eqn:rp_hash}
		P_{\mathbf{x}\sim\mathcal{X}}\left[\sgn(\mathbf{x}^T\mathbf{w}_m)=\sgn(\mathbf{x}^T\mathbf{w}_n)\right] = 1 - \frac{\theta}{\pi} = \frac{1+\vartheta}{2},
	\end{align}
	which is equivalent to Eq. \ref{eqn:g}. We then complete our proof.
\end{proof}

\subsection{Formulation}
Lemma \ref{lem:1} opens a door for us to estimate filter angles (\ie $\theta$) using filter responses (\ie $\mathbf{x}^T\mathbf{w}$). Due to linear transformation, both mean and variance in the $\theta$-space, \ie $\ang{90}$ and $\ang{0}$, are converted to $0$ in the $\vartheta$-space. Now based on the statistical relation between mean and variance, we can define our self-regularizer, $\mathcal{R}_{\vartheta}$, as follows:
\begin{align}\label{eqn:R}
\mathcal{R}_{\vartheta} \stackrel{def}{=} \sum_i\lambda_1\mathbb{E}_{\vartheta\sim\Theta_i}(\vartheta)^2 + \lambda_2\mathbb{E}_{\vartheta\sim\Theta_i}(\vartheta^2), 
\end{align}
where $\Theta_i, \forall i$ denotes the filter angle pool in the $\vartheta$-space at the $i$-th hidden layer, and $\lambda_1, \lambda_2\geq0$ are two predefined constants. Here we choose the least square loss for its simplicity, and other proper loss functions can be employed as well. In our experiments we choose $\lambda_1=100, \lambda_2=1$ by cross-validation using grid search. We observe $\lambda_2$ has much larger impact than $\lambda_1$ on performance, achieving similar accuracy with $\lambda_1\in[0,1000]$ and $\lambda_2\in[1, 10]$. This makes sense as the mean of an angular distribution in deep learning is close to $\ang{90}$ anyway, but the variance should be small.


\subsection{Implementation}
Computing $\vartheta$ in Eq. \ref{eqn:g} is challenging because of the expectation over all possible samples in $\mathcal{X}$. To approximate $\vartheta$, we introduce the notion of normalized approximate binary activation as follows:
\begin{defi}[Normalized Approximate Binary Activation\footnote{We tested different sign approximation functions such as softsign, and observed that their performances are very close. Therefore, by referring to \cite{chang2017deep} in this paper we utilize $\tanh$ only.} (NABA)]\label{def:NABA}
	Letting $\mathbf{w}\in\mathbb{R}^d$ be a vector and $\mathbf{X}\in\mathbb{R}^{d\times D}$ be a projection matrix, then we define an NABA vector, $\mathbf{z}\in\mathbb{R}^D$, for $\mathbf{w}$ as
	\begin{align}\label{eqn:NABA}
		\mathbf{z} = \frac{\tanh\left(\gamma\mathbf{X}^T\mathbf{w}\right)}{\|\tanh(\gamma\mathbf{X}^T\mathbf{w})\|} \Rightarrow \vartheta(\mathbf{w}_m, \mathbf{w}_n)\approx\mathbf{z}_m^T\mathbf{z}_n, \forall m\neq n
	\end{align}
	where $\tanh$ is an entry-wise function, 
	and $\gamma\geq0$ is a scalar so that $\lim_{\gamma\rightarrow+\infty}\tanh\left(\gamma x\right)=\sgn\left(x\right), \forall x\in\mathbb{R}$ as used in~\cite{chang2017deep}.
\end{defi}

\begin{figure}[th]
	\begin{center}
		\includegraphics[width=\linewidth]{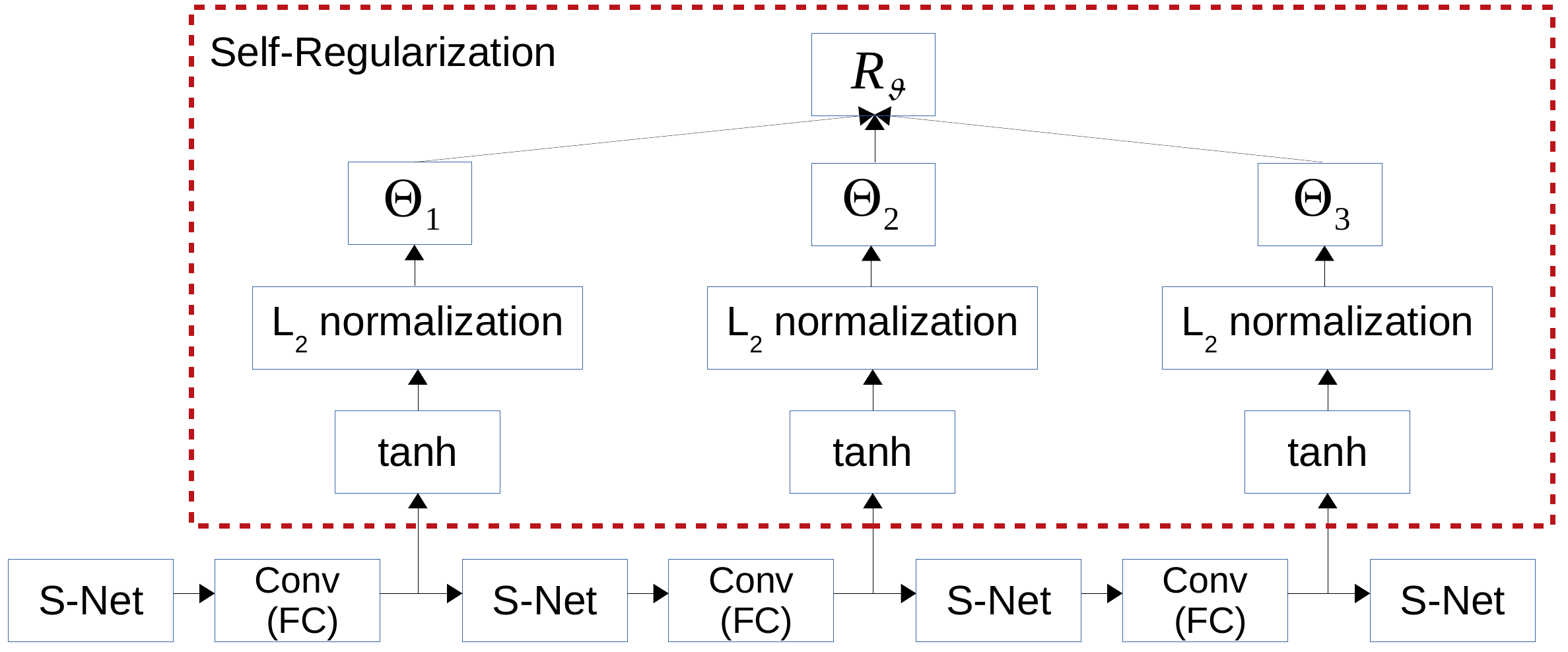}
		\caption{An example of integration of our self-regularization as a plug-in on a three hidden-layer DNN as backbone for training. Here ``S-Net'' denotes a sub-network consisting of non-hidden layers.}
		\label{fig:reg_net}
	\end{center}
\end{figure}

Based on the consideration of accuracy and running speed, in our experiments we set $\gamma=10, D=16$ for Eq. \ref{eqn:NABA}. Larger $\gamma$ brings more difficulty in optimization, as demonstrated in \cite{chang2017deep}. Larger $D$ does approximate the expectation in Eq.~\ref{eqn:g} better, but has marginal effect on training loss and test accuracy, as well as leading to higher computational burden.

Implementing Eq. \ref{eqn:NABA} using networks is simple, as illustrated in Fig. \ref{fig:reg_net} where $\mathbf{X}$ denotes the input features to each convolutional (conv) or fully-connected (FC) hidden layer and $\mathbf{w}$ denotes the weights of a filter at the hidden layer. 
We work on the activation of hidden layers and transform feature outputs into a two-dimensional matrix of channel numbers by the neuron dimension, the product of batch size, and feature map's weight and height in the convolution case. Then we take samples along the neuron dimension. Accordingly, $z$ is a vector with the size of sampling number so that the filter angle approximation is workable.
Thus, both FC and conv can follow the way of $\mathbf{z}_m^T\mathbf{z}_n$. In this way, our regularization can be easily integrated into an arbitrary network seamlessly as a plug-in so that the network itself can automatically regularize its weights internally and explicitly, leading to self-regularization. 

Compared with existing OR algorithms, our self-regularization takes the advantage of the dependency between input features and filters so that the weights are learned more specifically to better fit the data. From this perspective, our self-regularization is similar to a family of batch normalization algorithms. As a consequence, we may not need any external regularization to help training.  

\begin{figure*}[t]
    \subfigure{\includegraphics[width=0.25\linewidth]{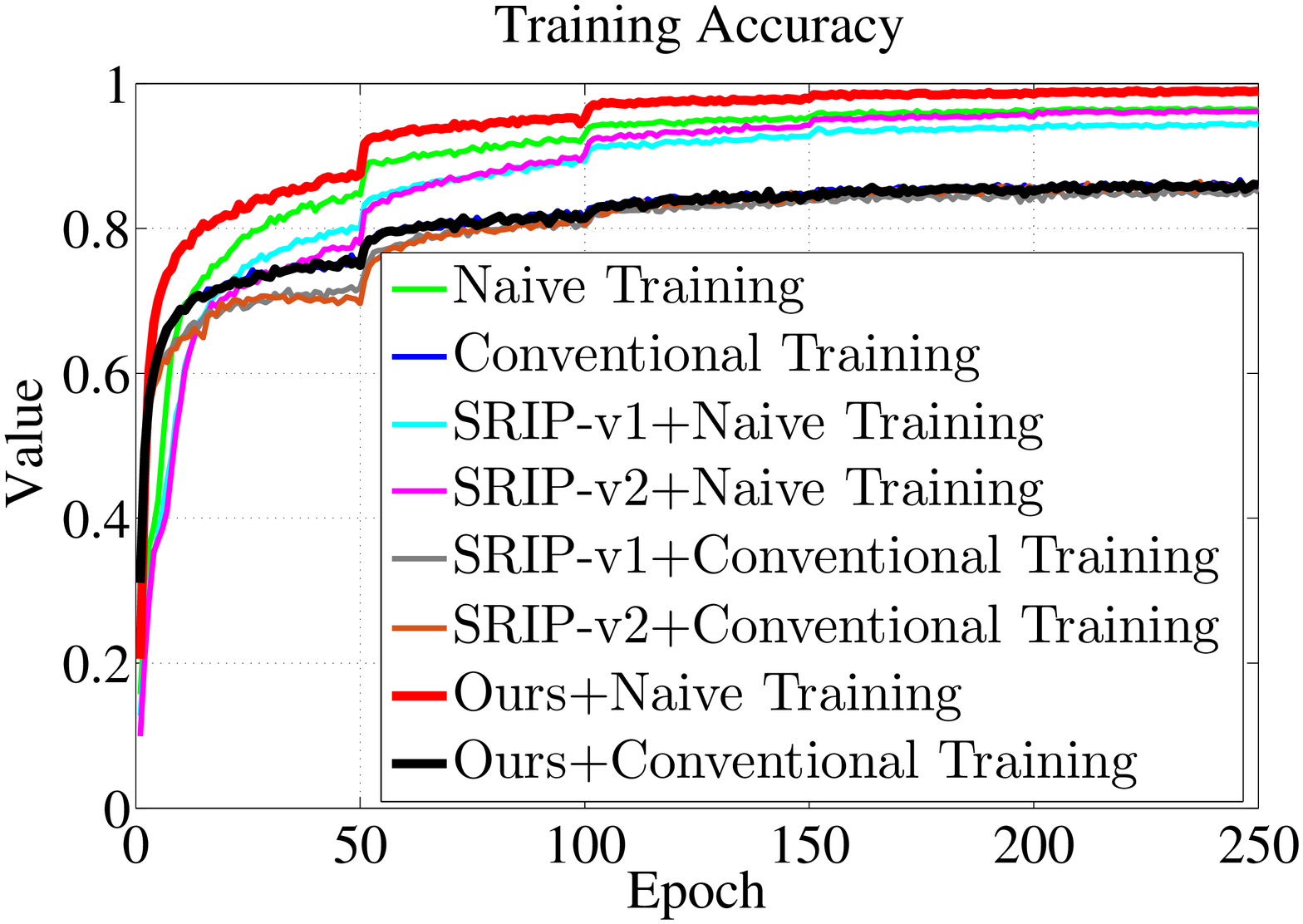}}\hfill
    \subfigure{\includegraphics[width=0.25\linewidth]{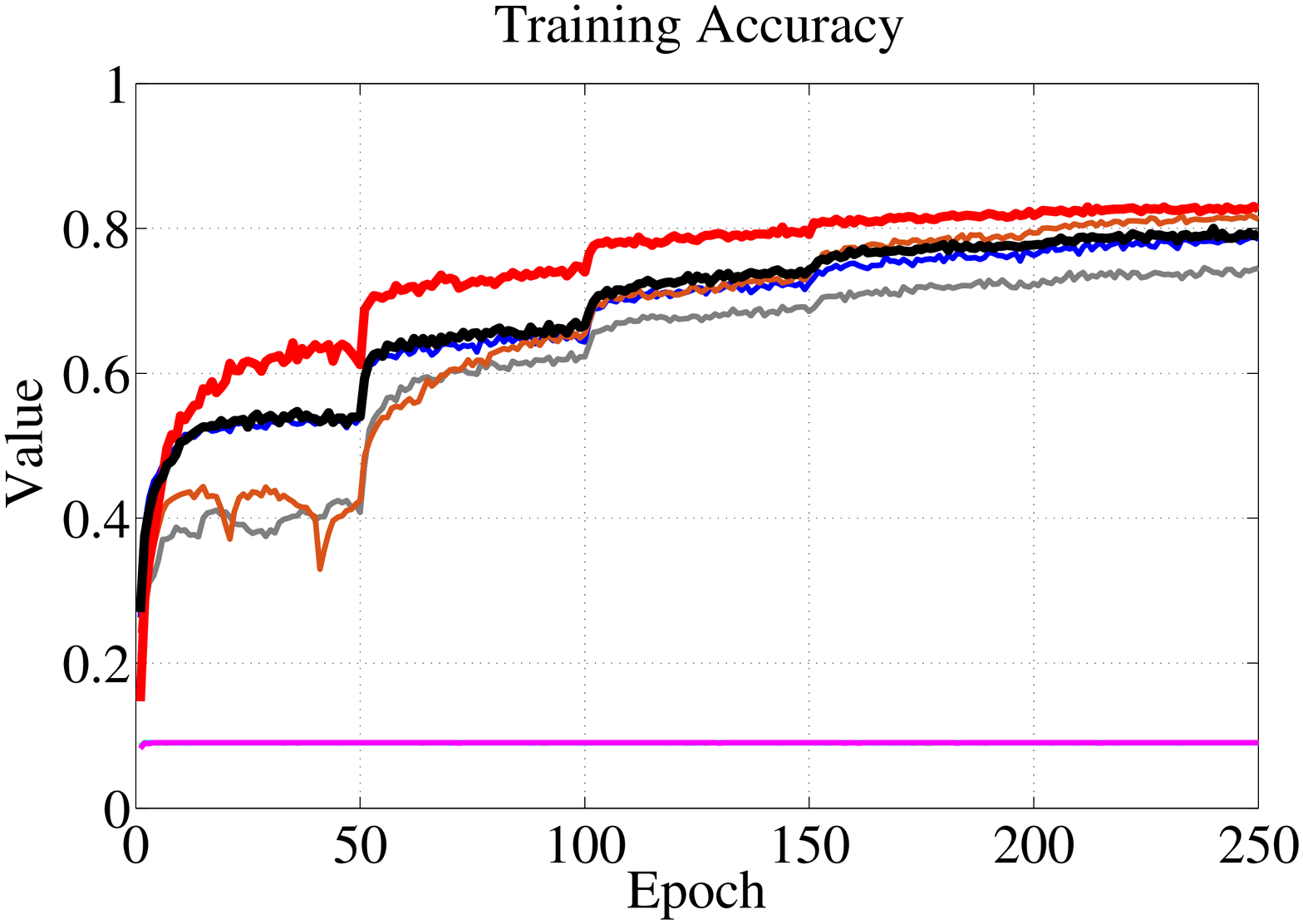}}\hfill
    \subfigure{\includegraphics[width=0.25\linewidth]{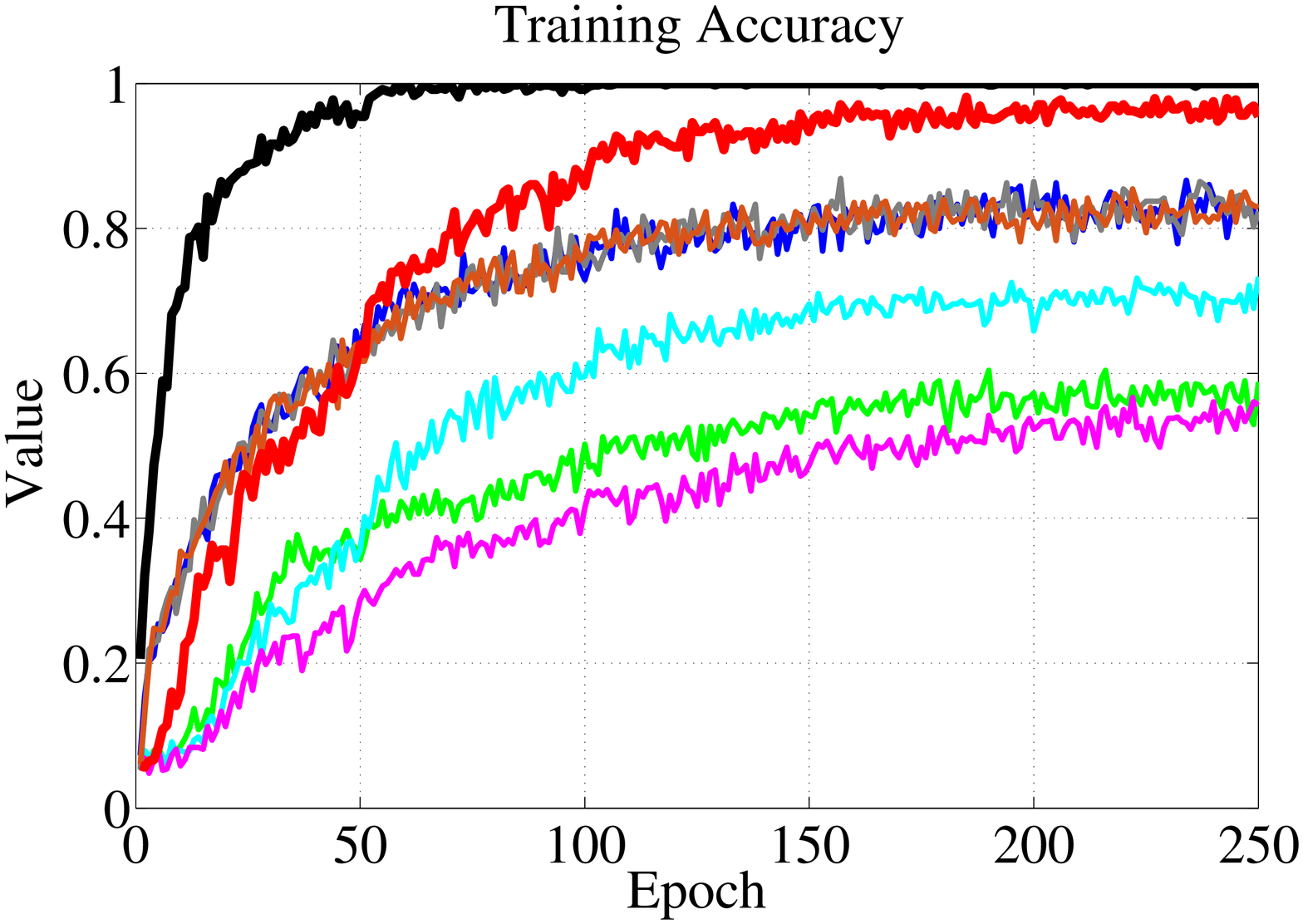}}\hfill
    \subfigure{\includegraphics[width=0.25\linewidth]{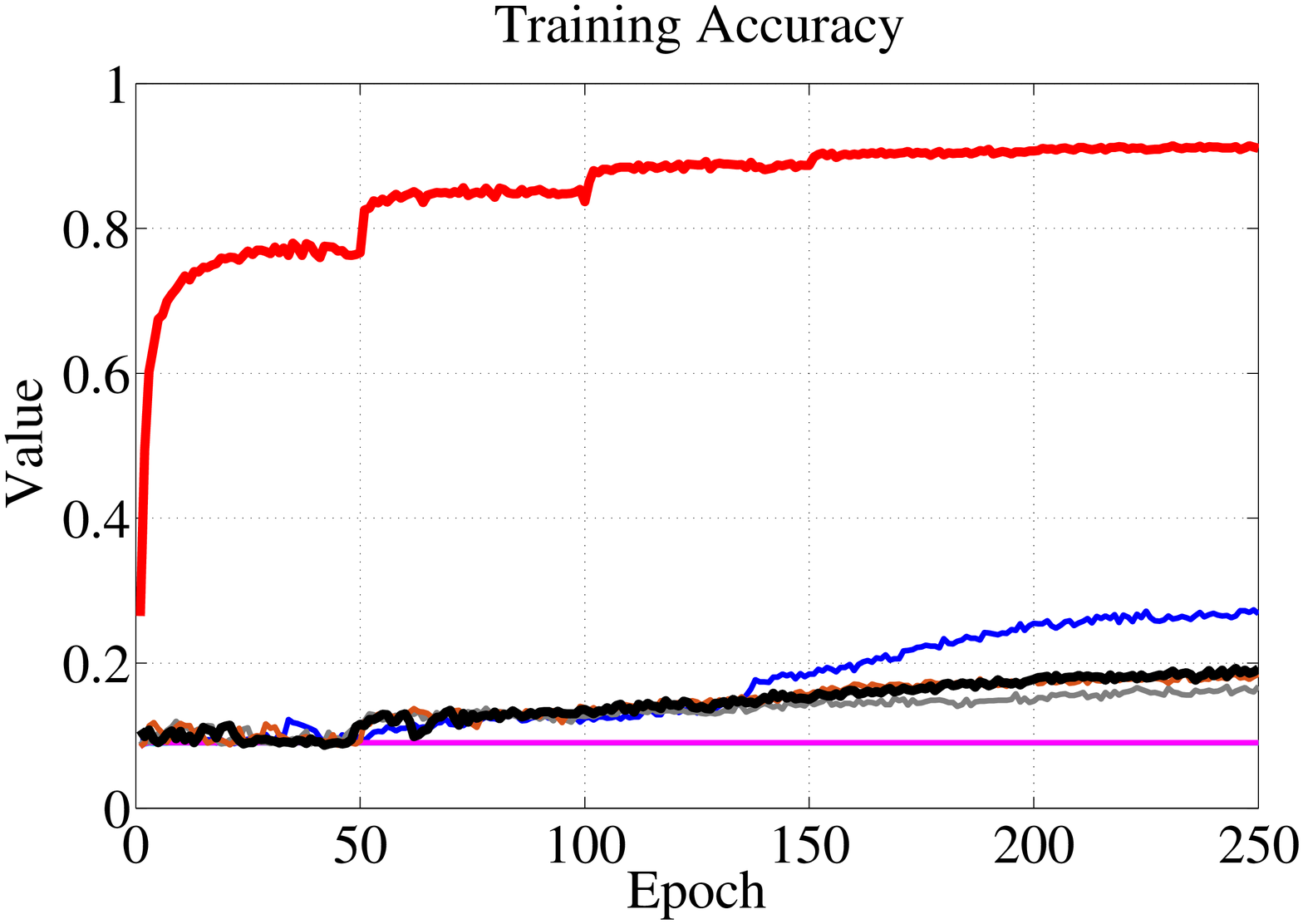}}\\[-1.5ex]
    \subfigure{\includegraphics[width=0.25\linewidth]{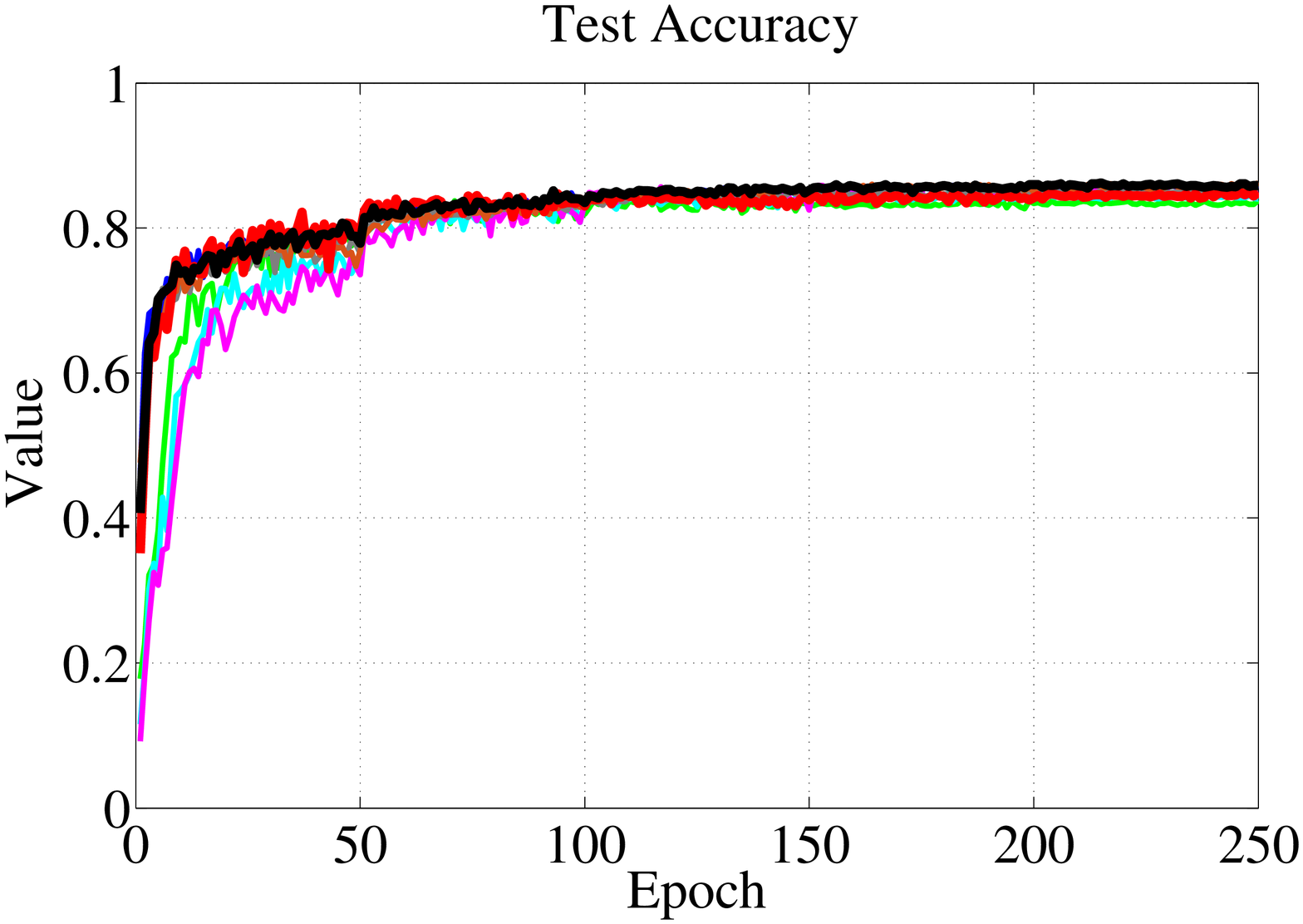}}\hfill
    \subfigure{\includegraphics[width=0.25\linewidth]{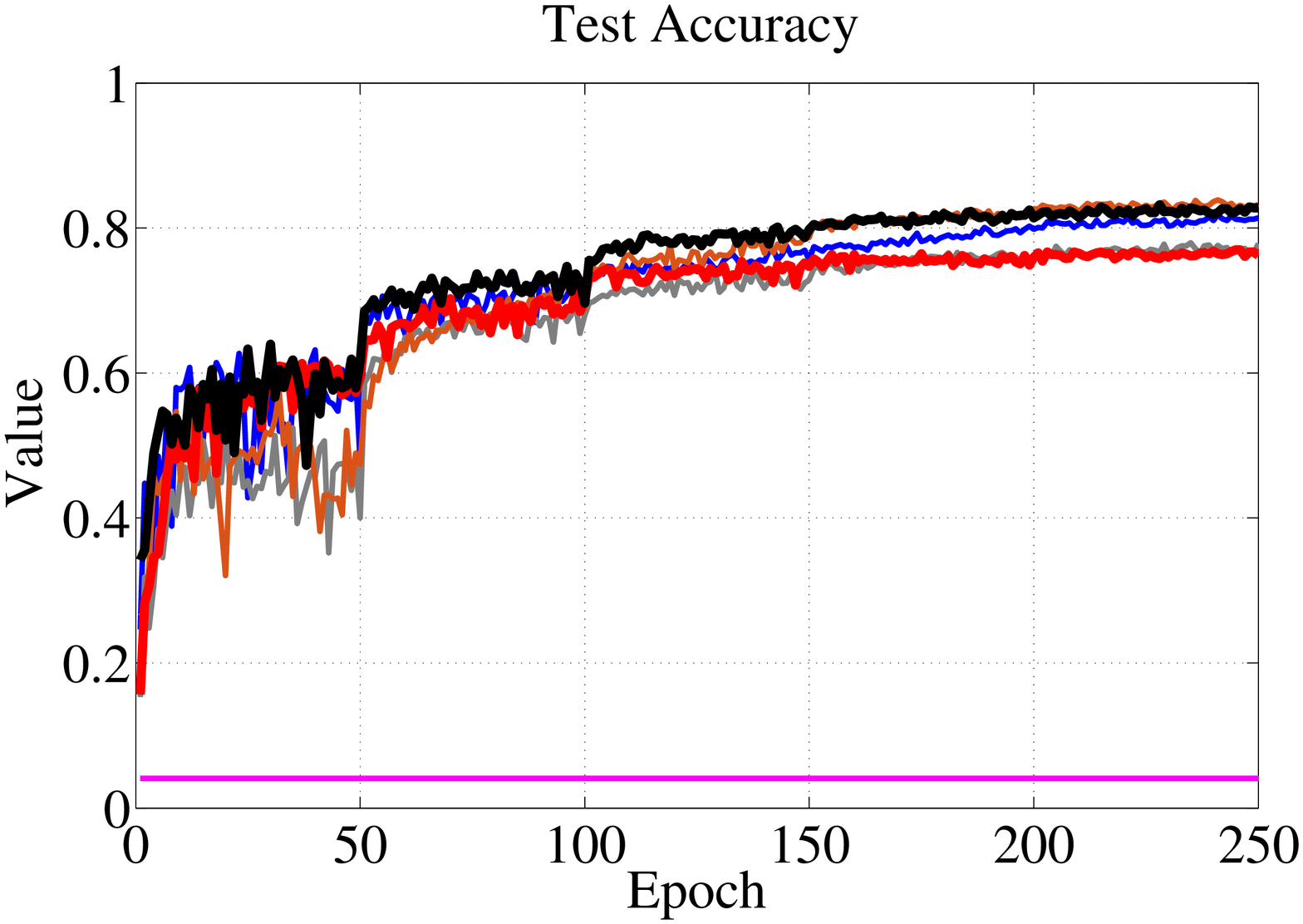}}\hfill
    \subfigure{\includegraphics[width=0.25\linewidth]{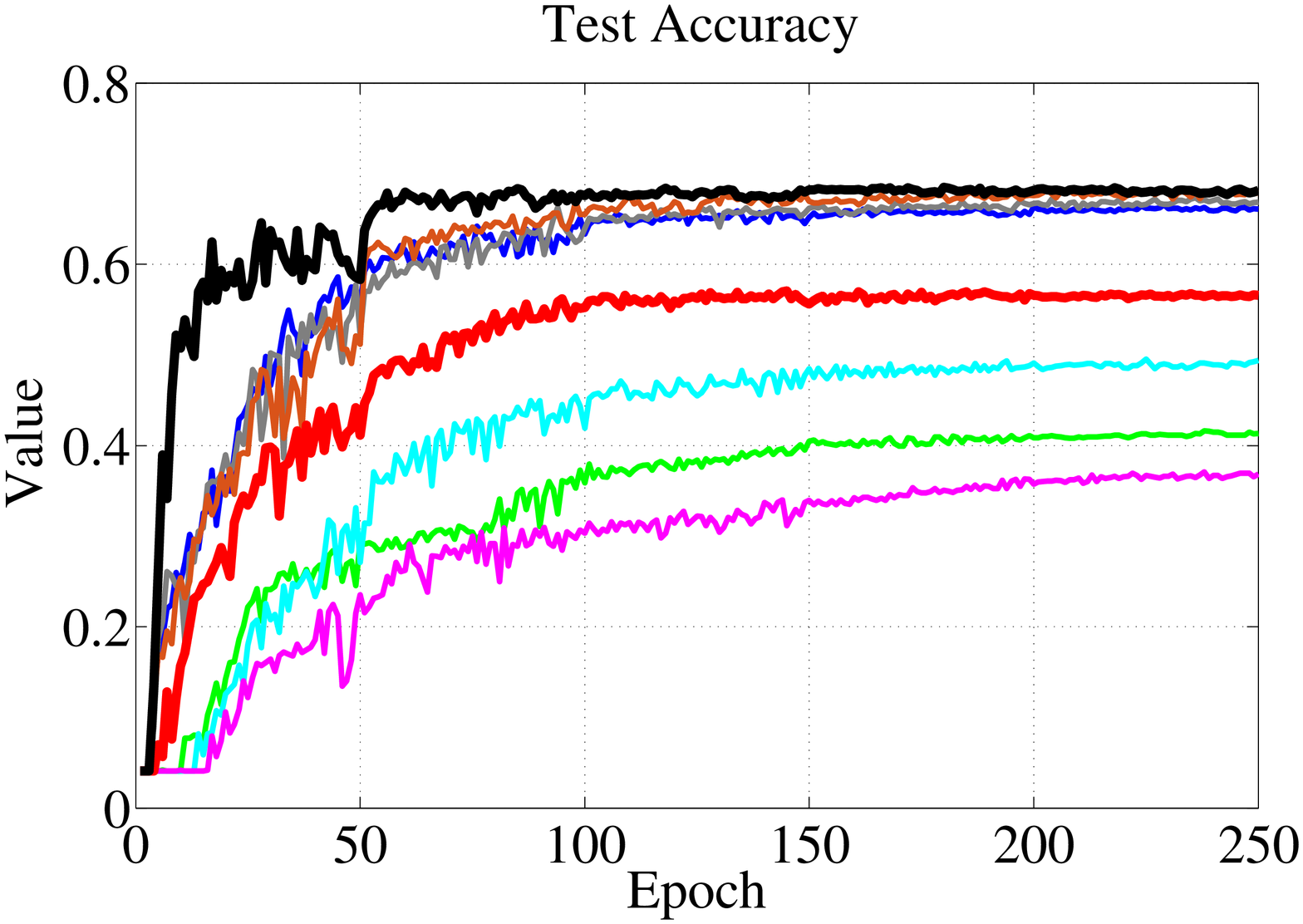}}\hfill
    \subfigure{\includegraphics[width=0.25\linewidth]{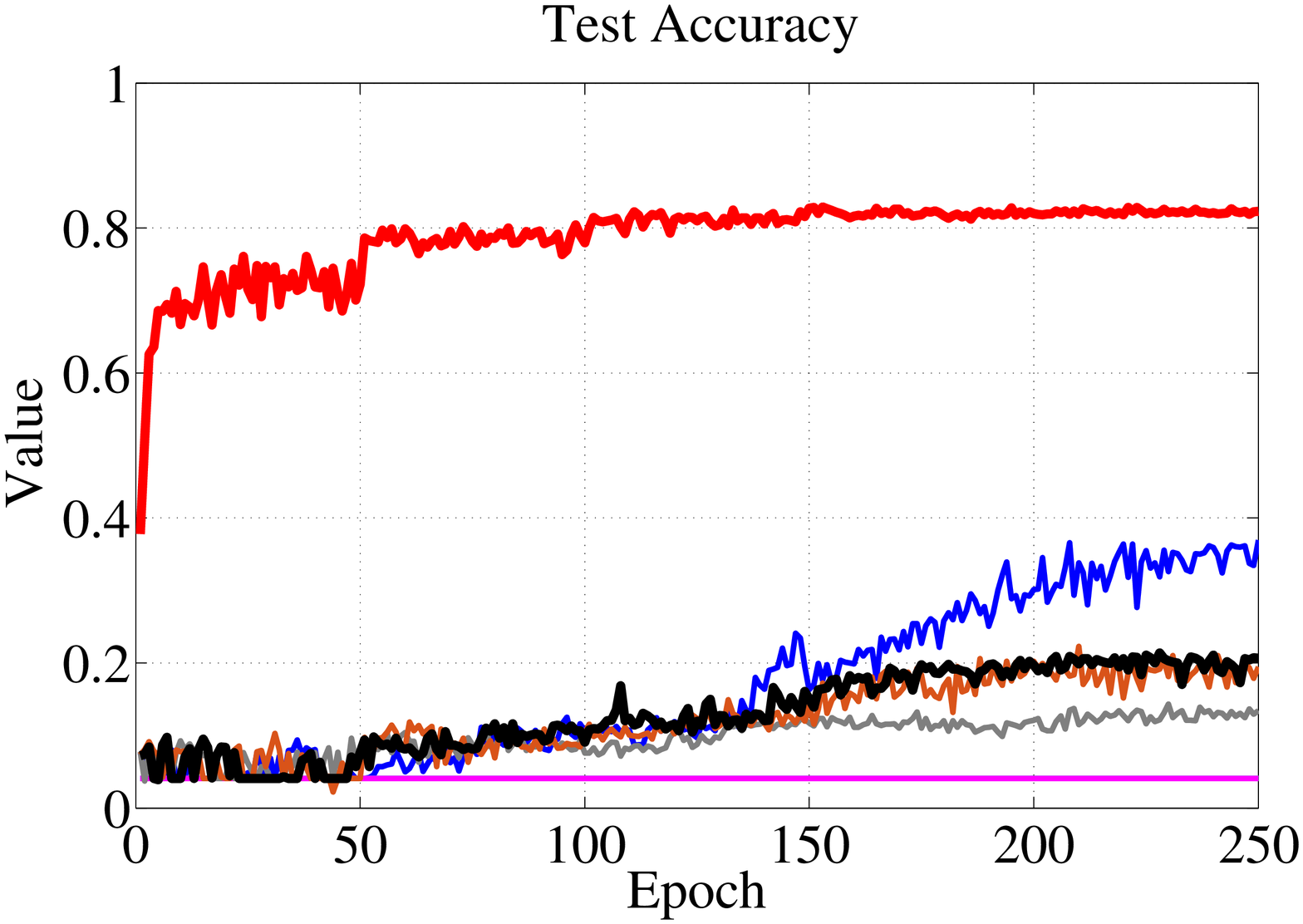}}\\[-1.5ex]
    \subfigure{\includegraphics[width=0.25\linewidth]{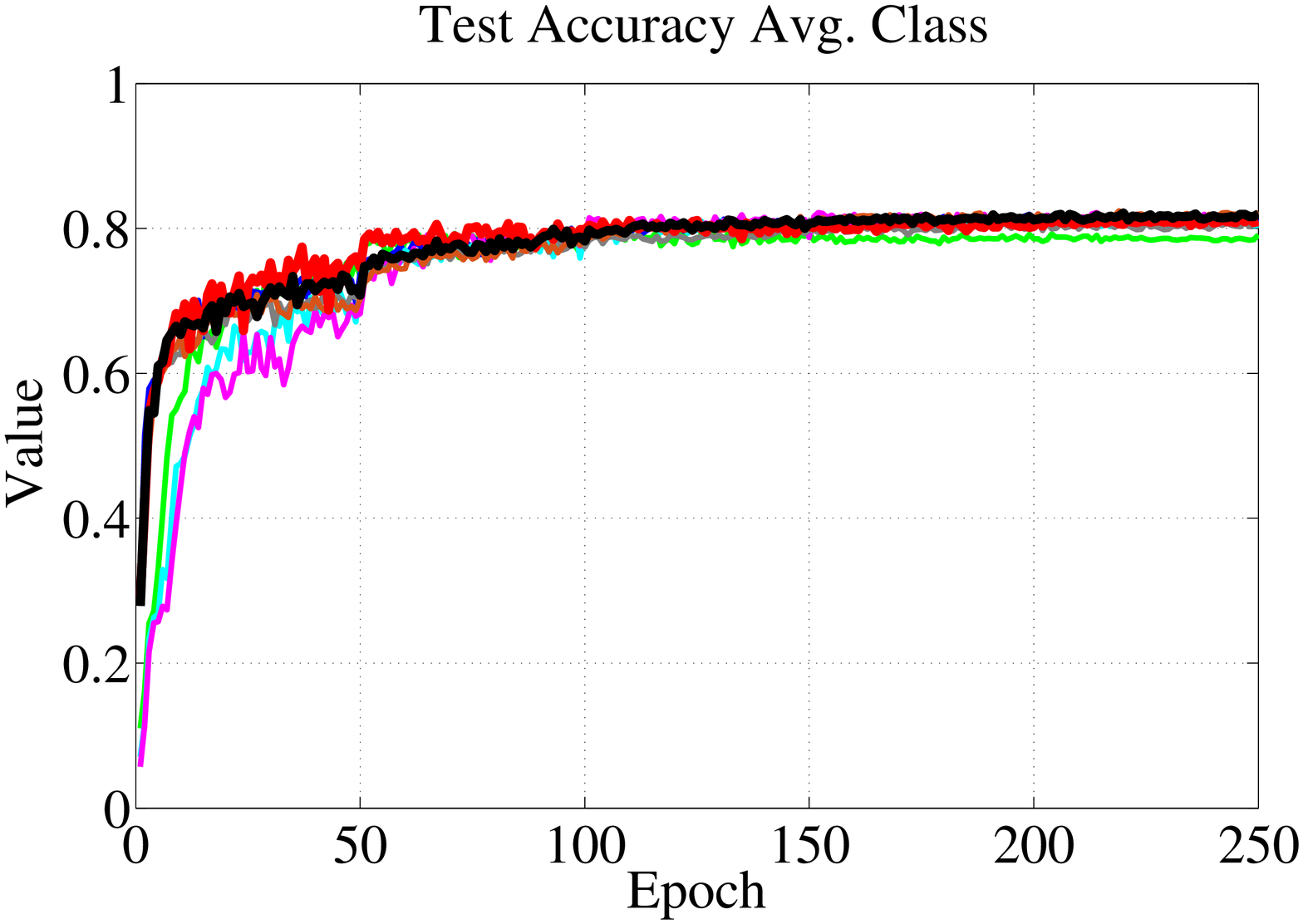}}\hfill
    \subfigure{\includegraphics[width=0.25\linewidth]{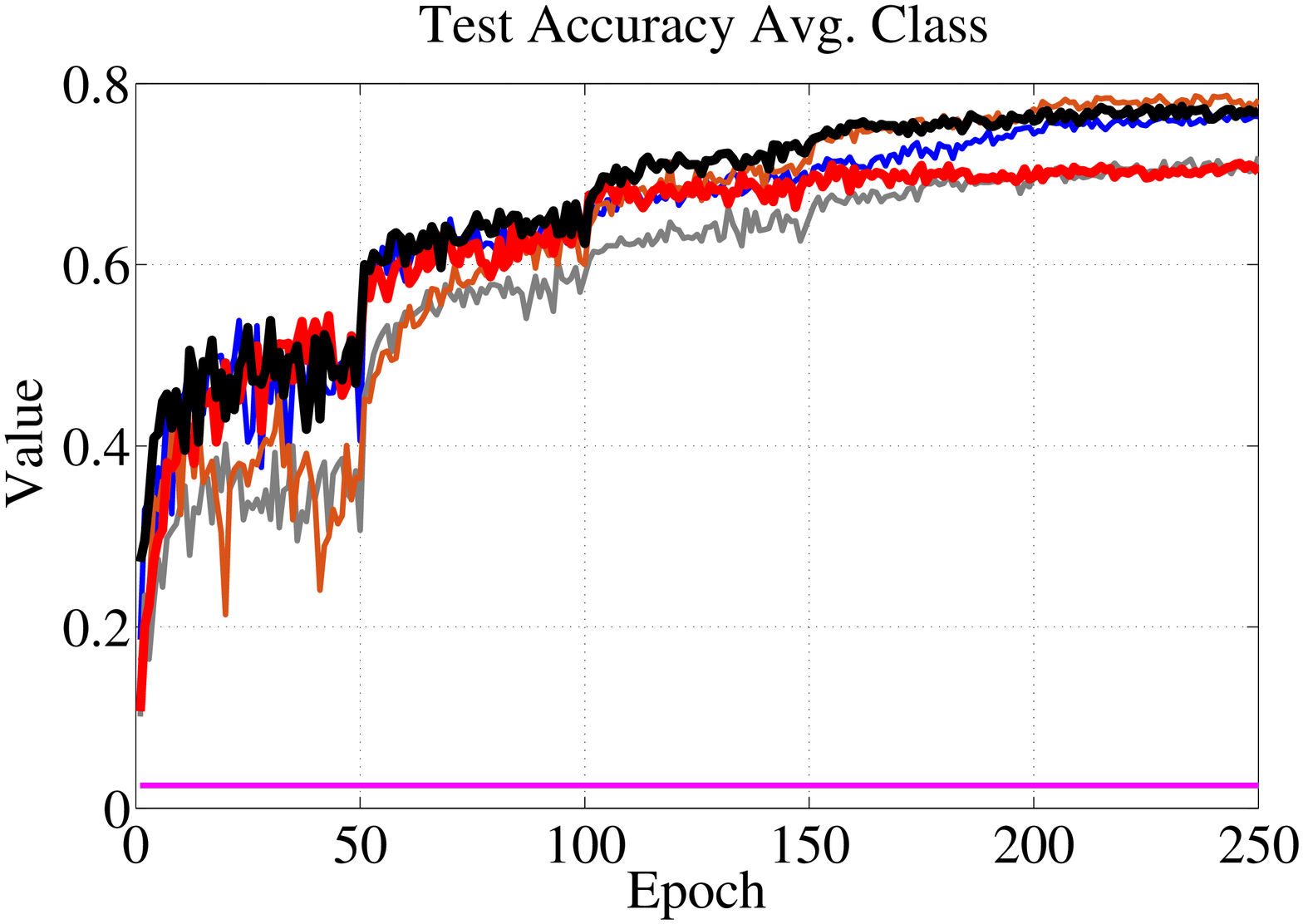}}\hfill
    \subfigure{\includegraphics[width=0.25\linewidth]{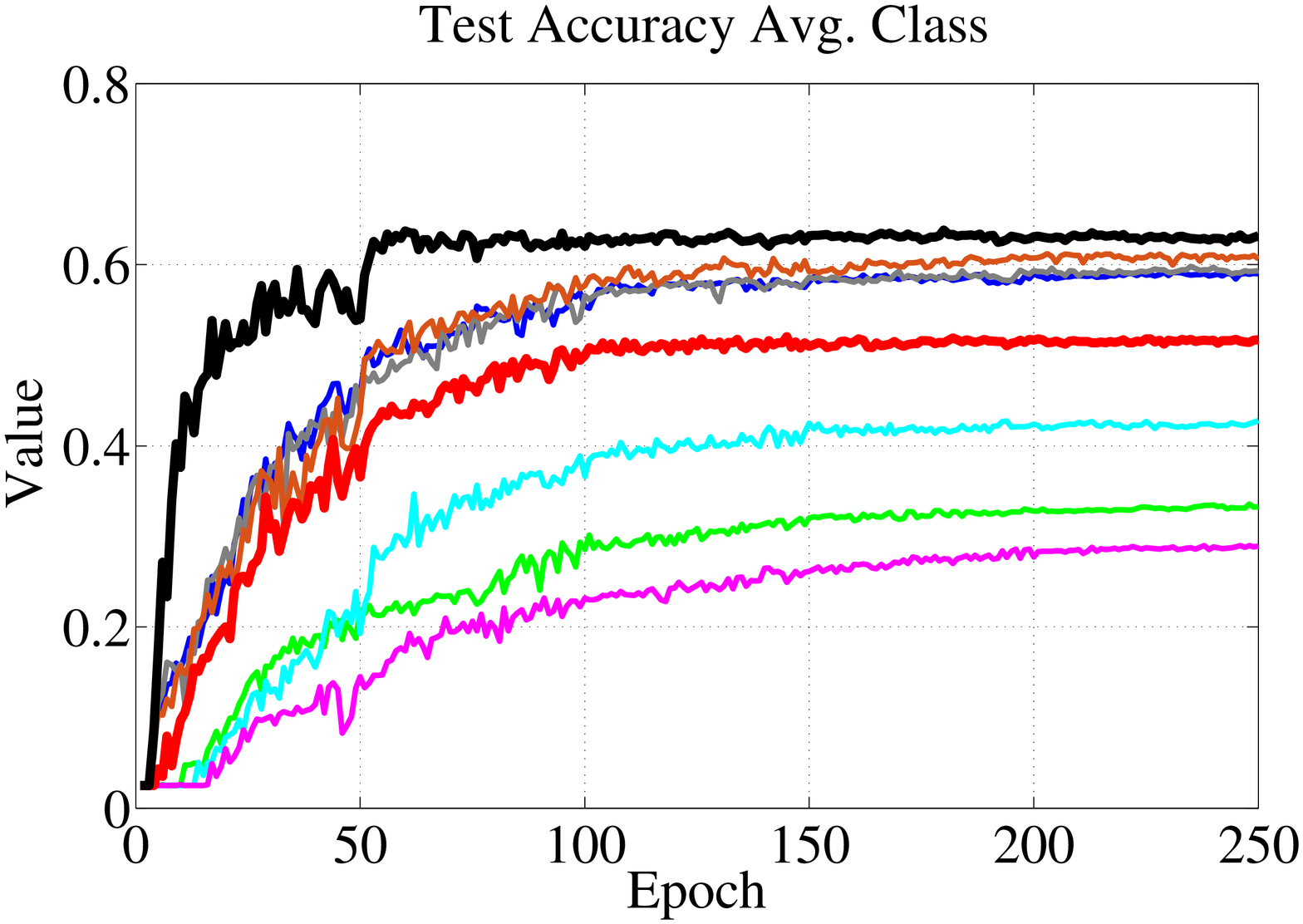}}\hfill
    \subfigure{\includegraphics[width=0.25\linewidth]{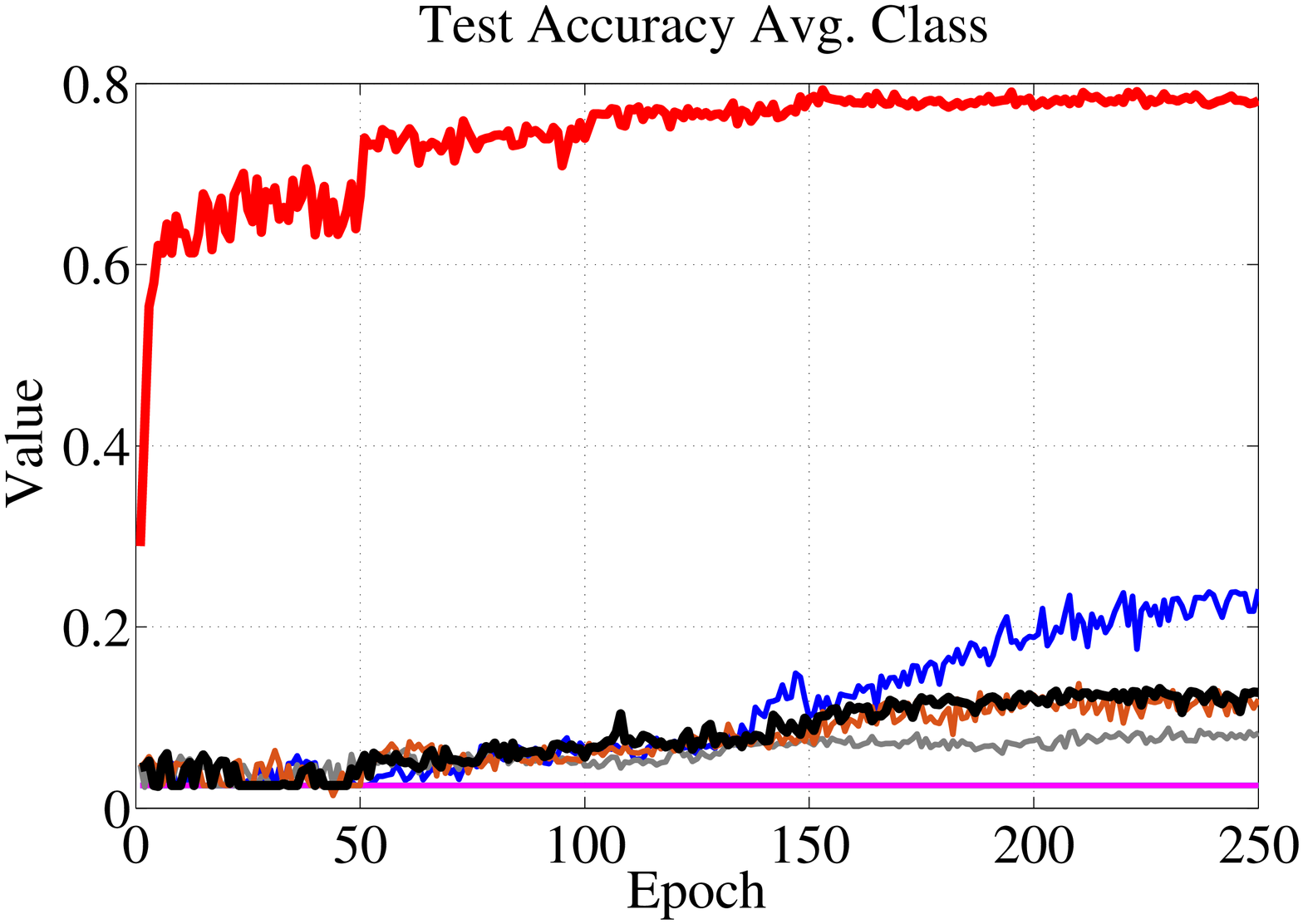}}
    \caption{Result comparison on ModelNet40: {\bf (left->right)} default setting, $lr=0.01$, $ts=440$, and $bs=2$.}
	\label{fig:4_cases_modelnet40}
\end{figure*}

\begin{figure*}[t]
    \subfigure{\includegraphics[width=0.25\linewidth]{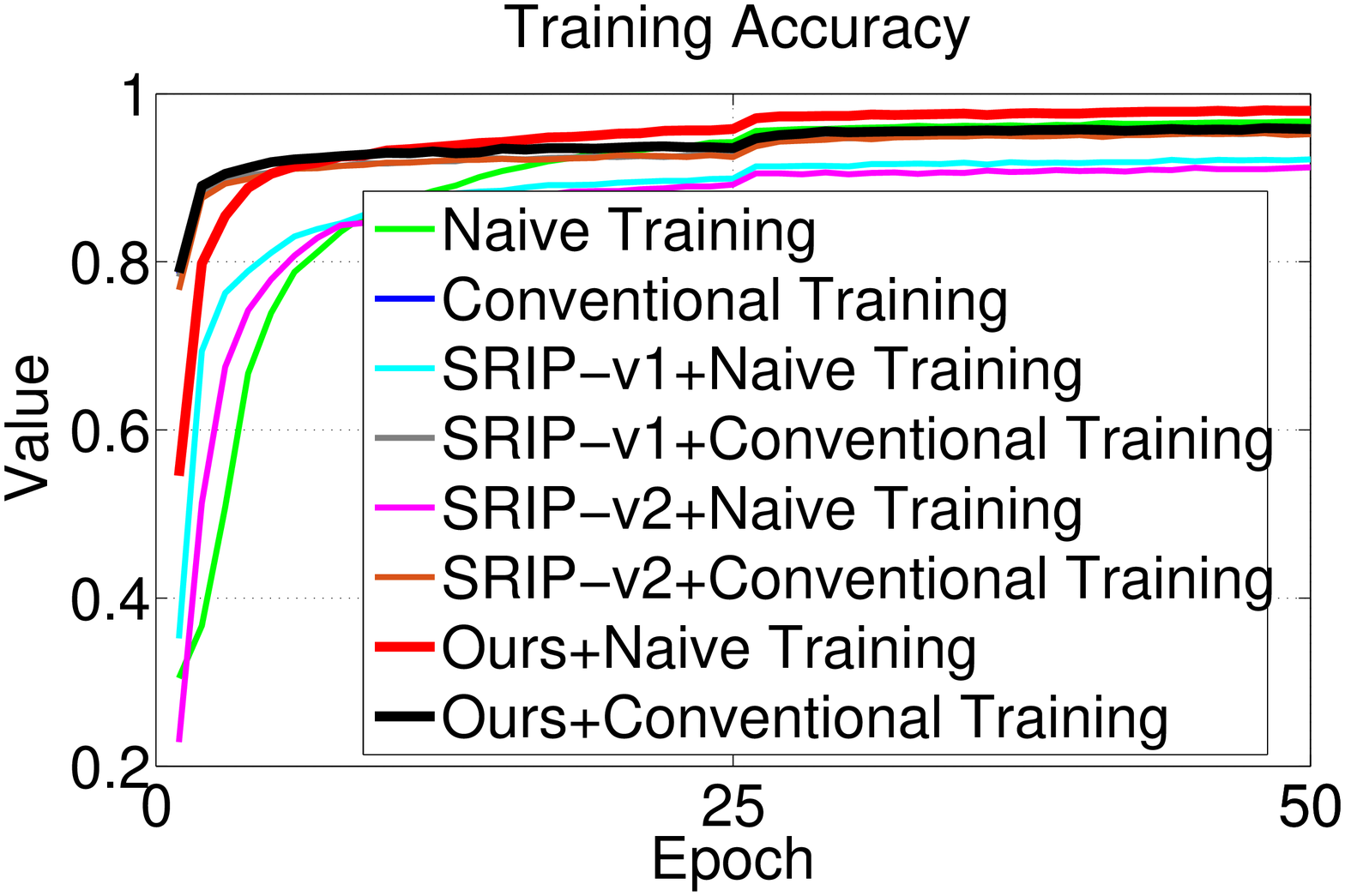}}\hfill
    \subfigure{\includegraphics[width=0.25\linewidth]{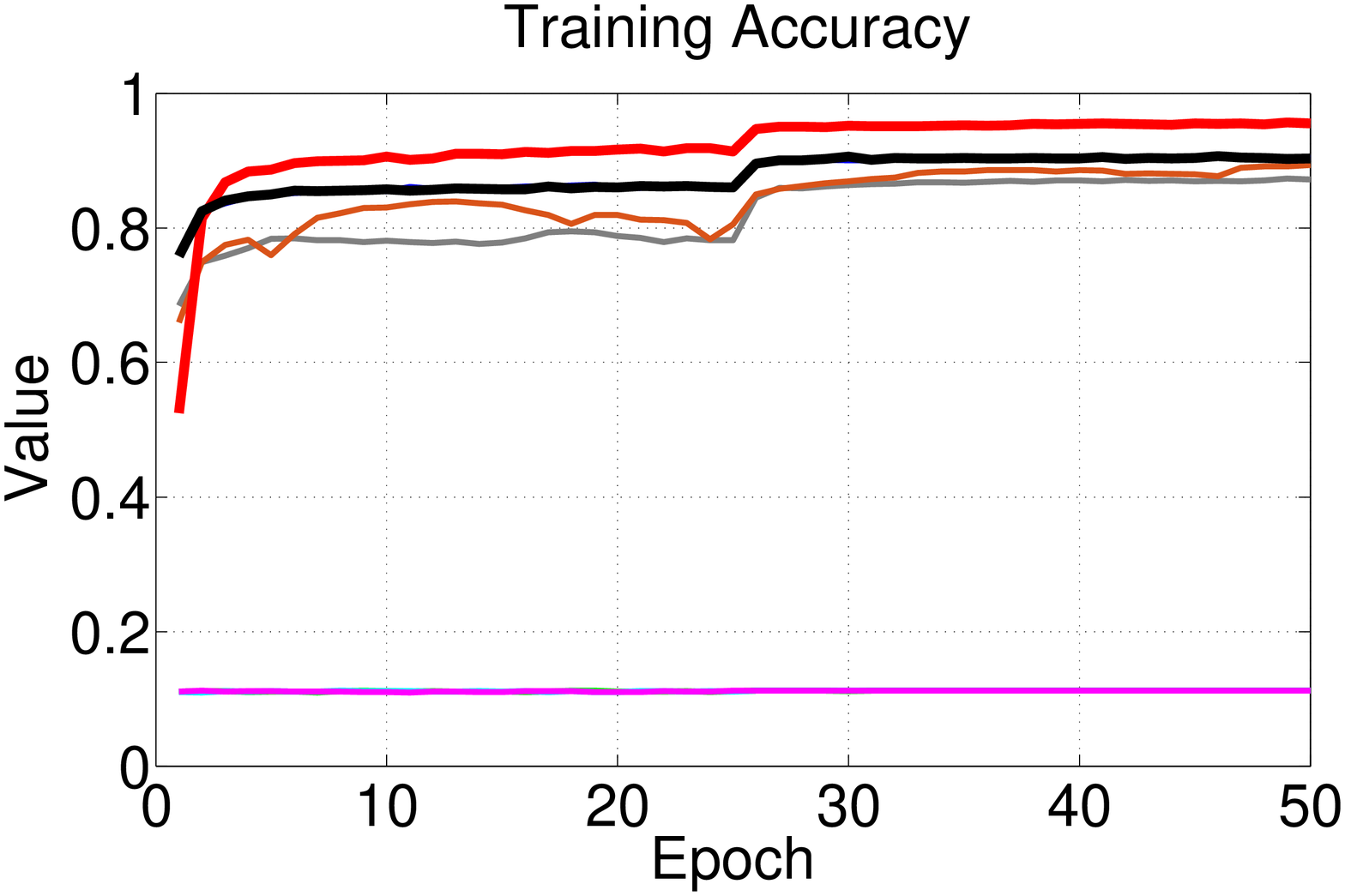}}\hfill
    \subfigure{\includegraphics[width=0.25\linewidth]{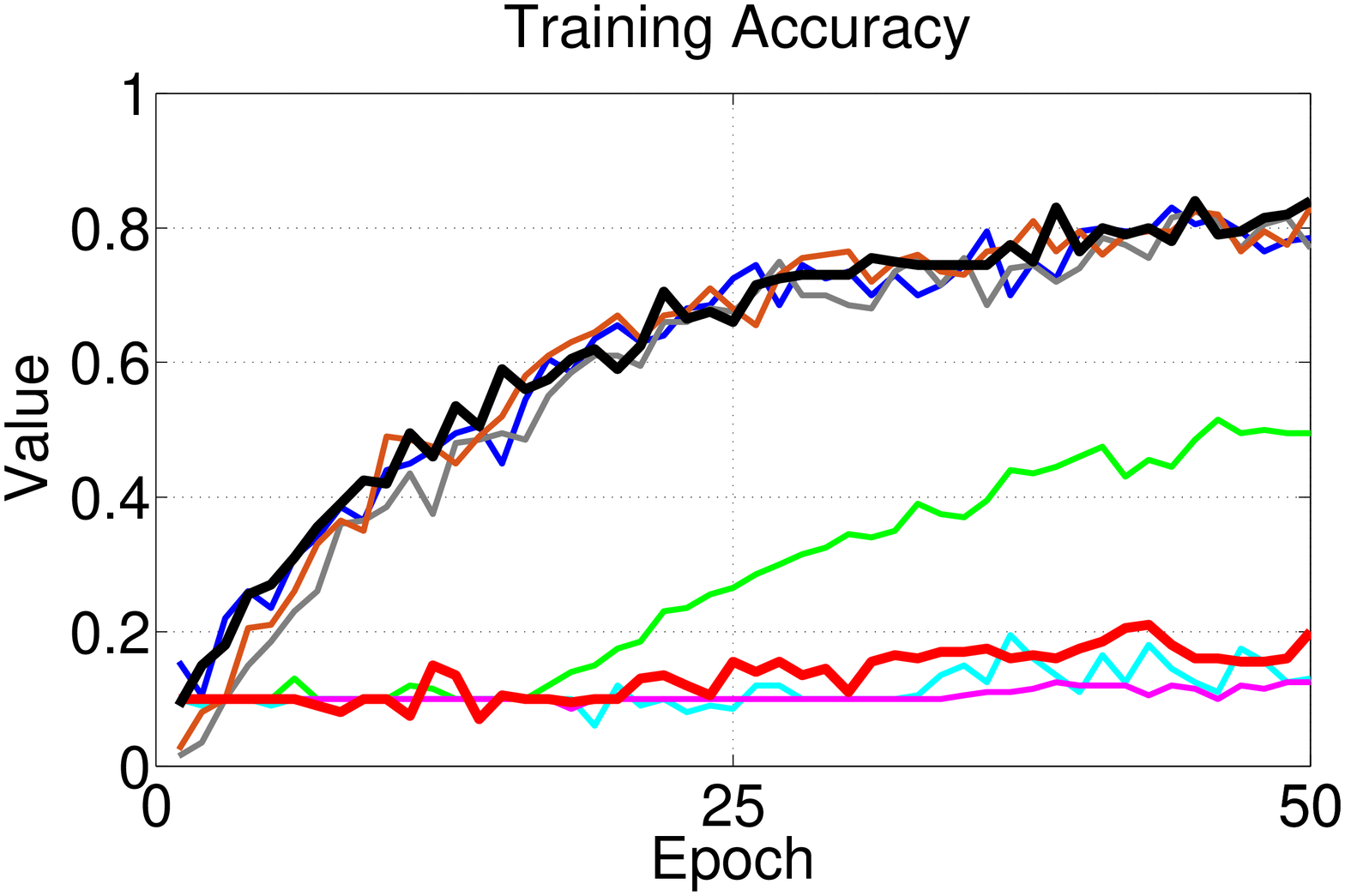}}\hfill
    \subfigure{\includegraphics[width=0.25\linewidth]{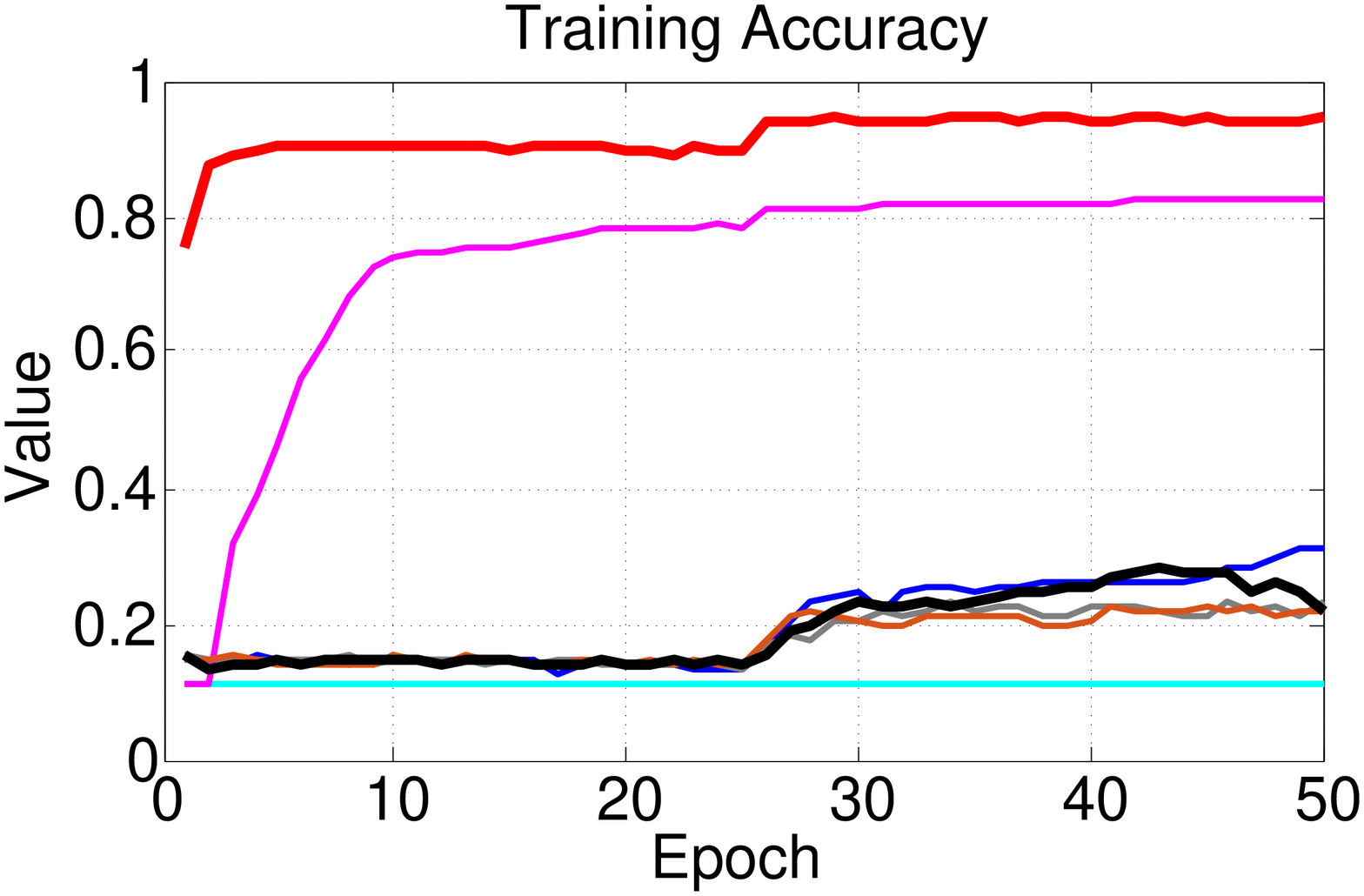}}\hfill
    \subfigure{\includegraphics[width=0.25\linewidth]{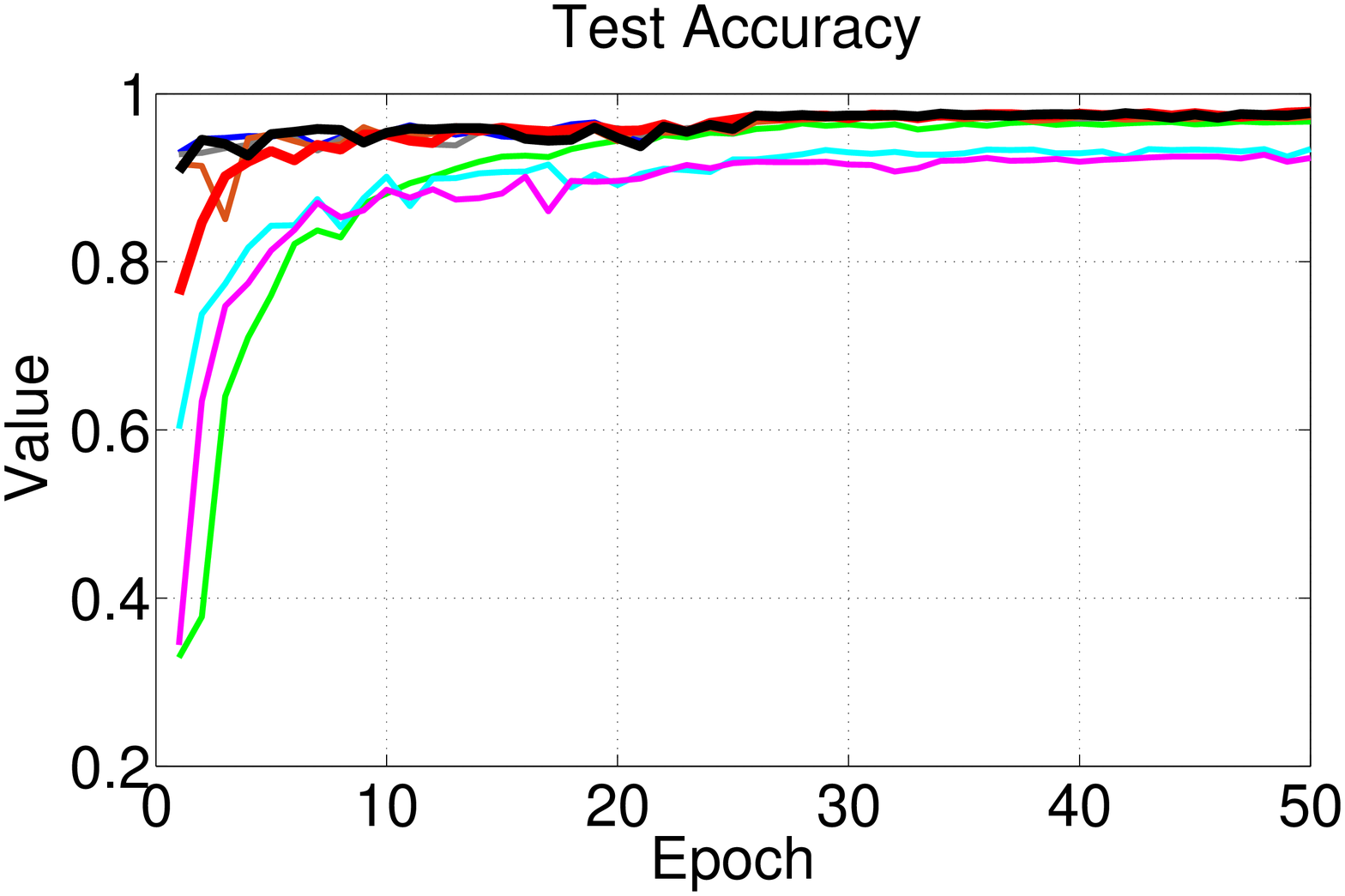}}\hfill
    \subfigure{\includegraphics[width=0.25\linewidth]{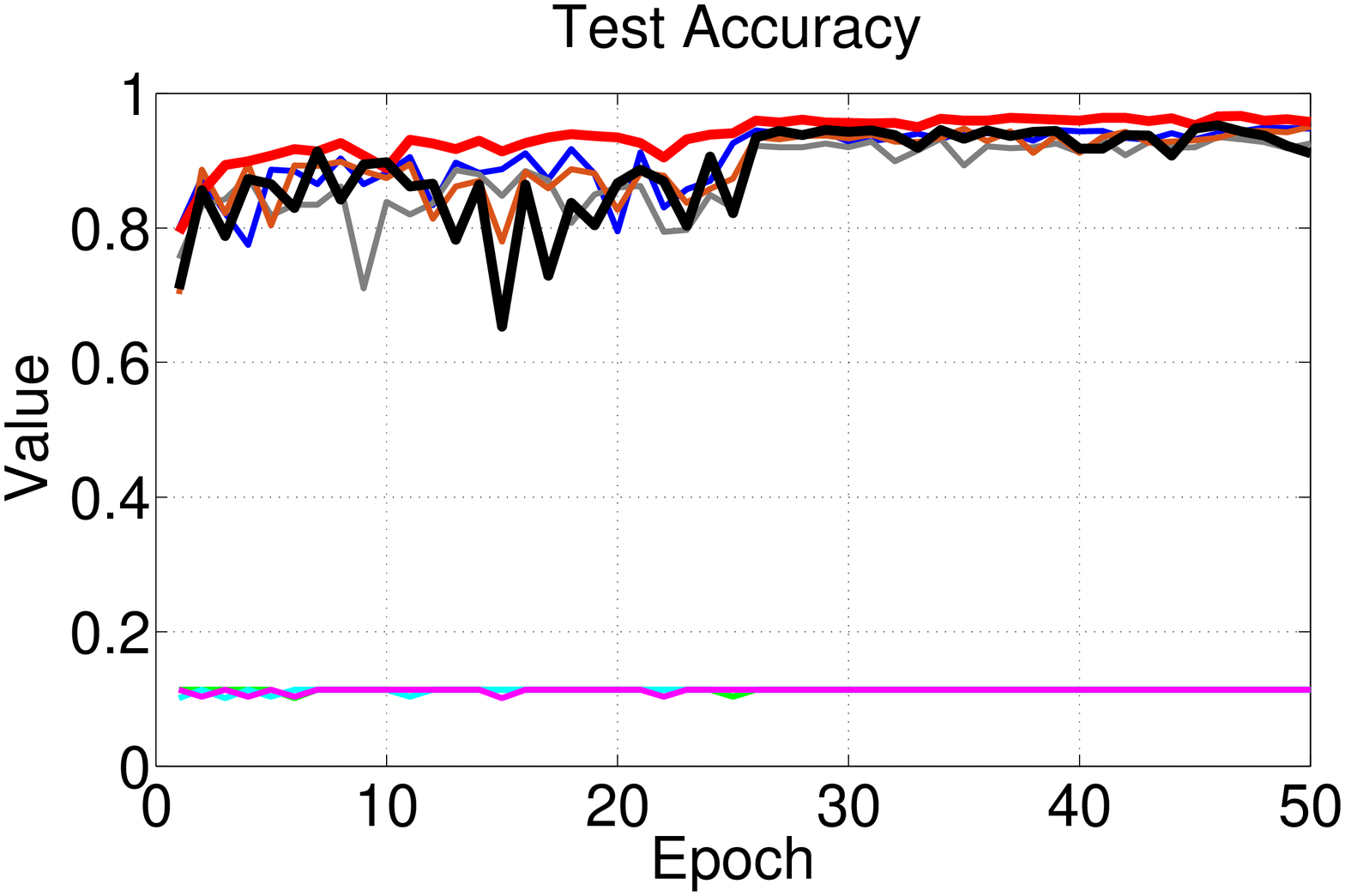}}\hfill
    \subfigure{\includegraphics[width=0.25\linewidth]{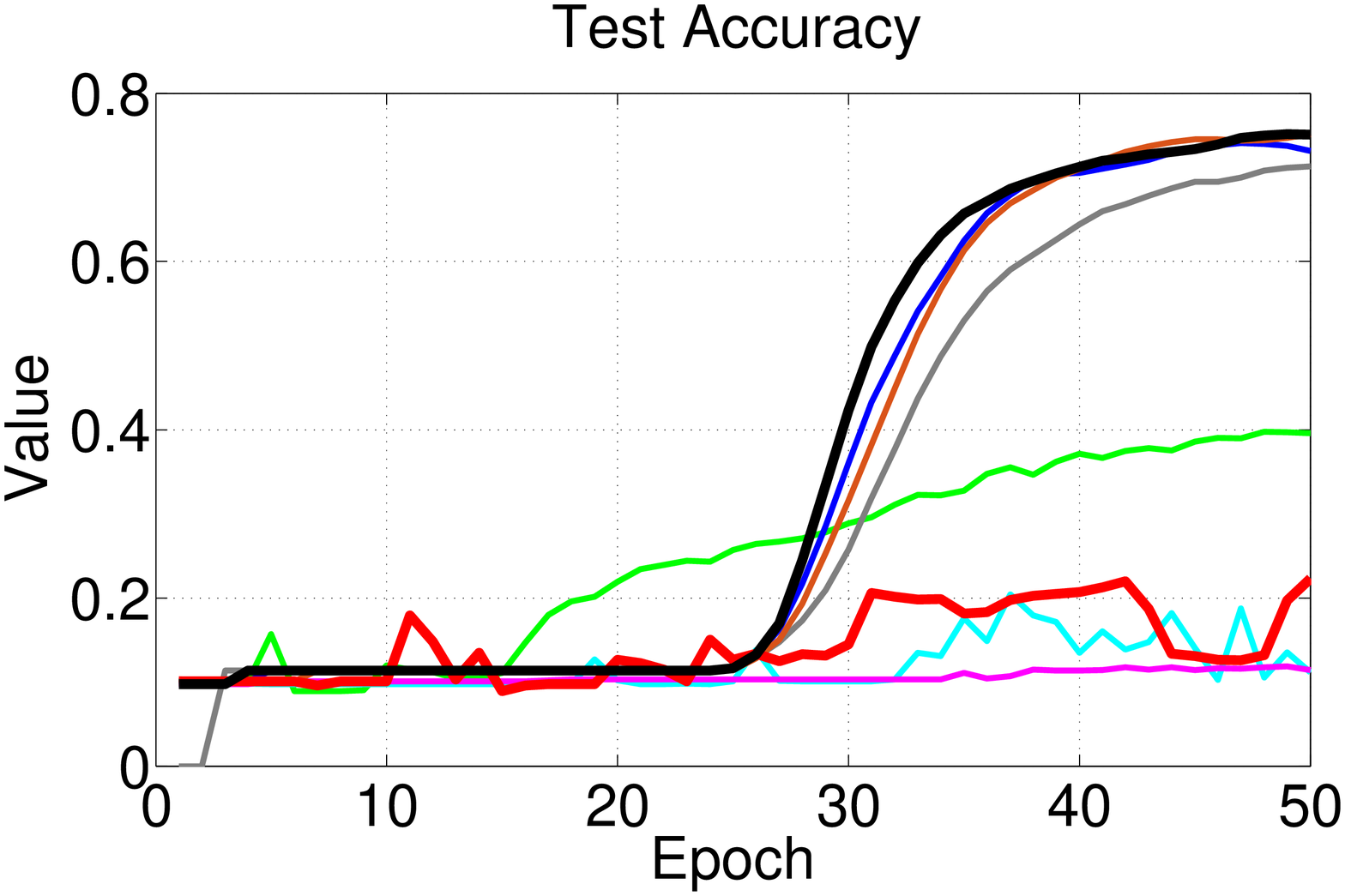}}\hfill
    \subfigure{\includegraphics[width=0.25\linewidth]{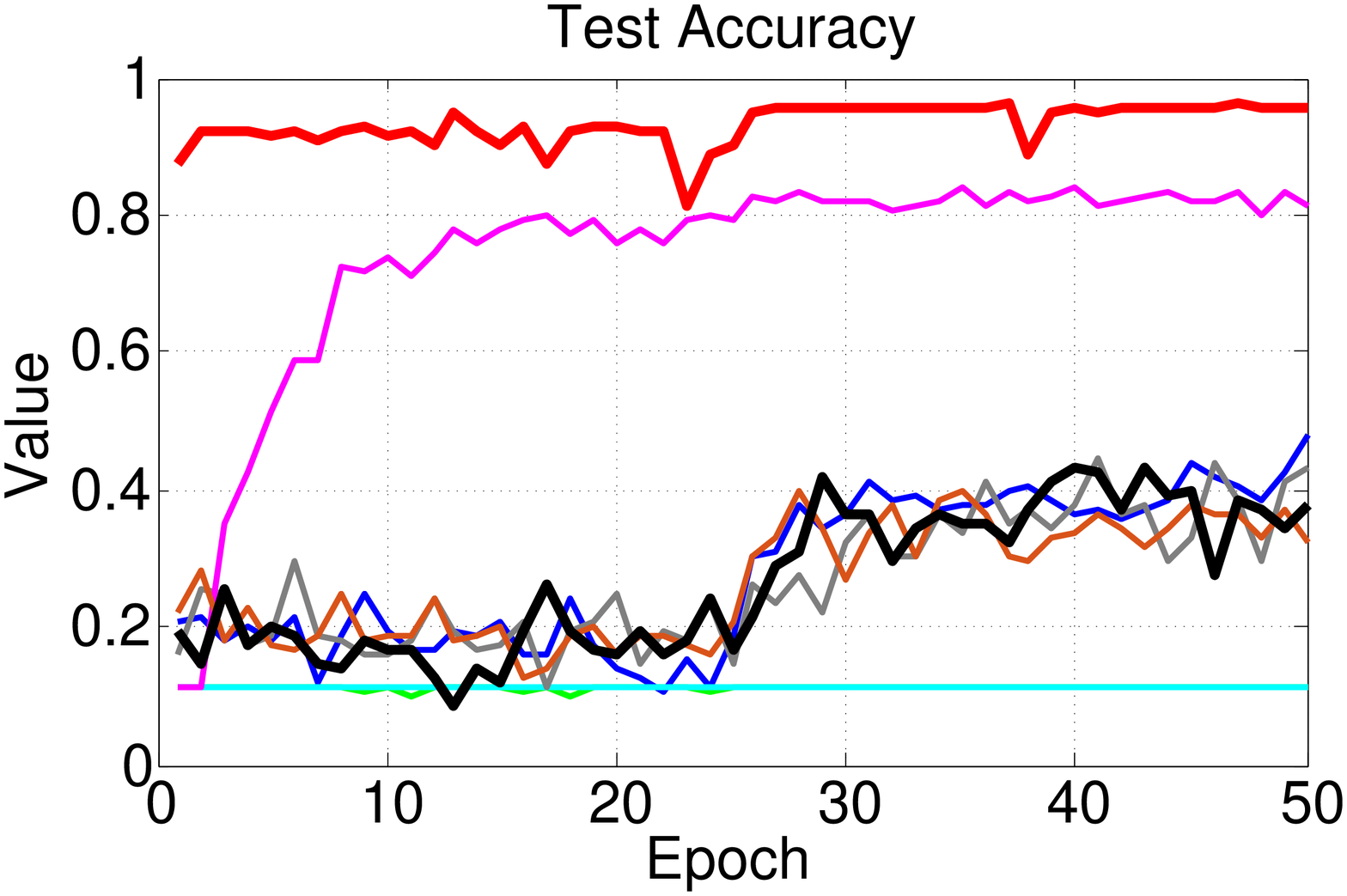}}\hfill
    \subfigure{\includegraphics[width=0.25\linewidth]{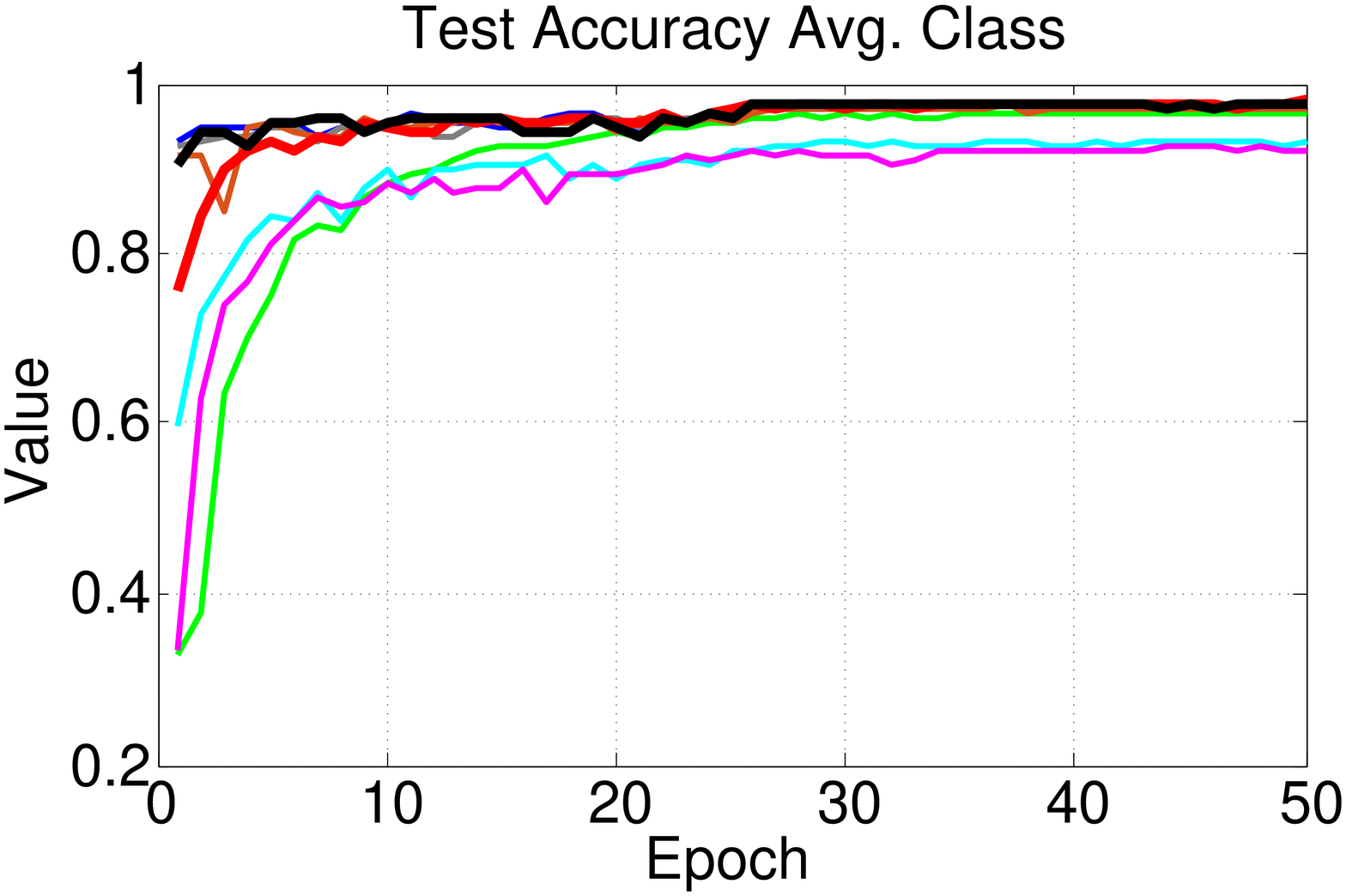}}\hfill
    \subfigure{\includegraphics[width=0.25\linewidth]{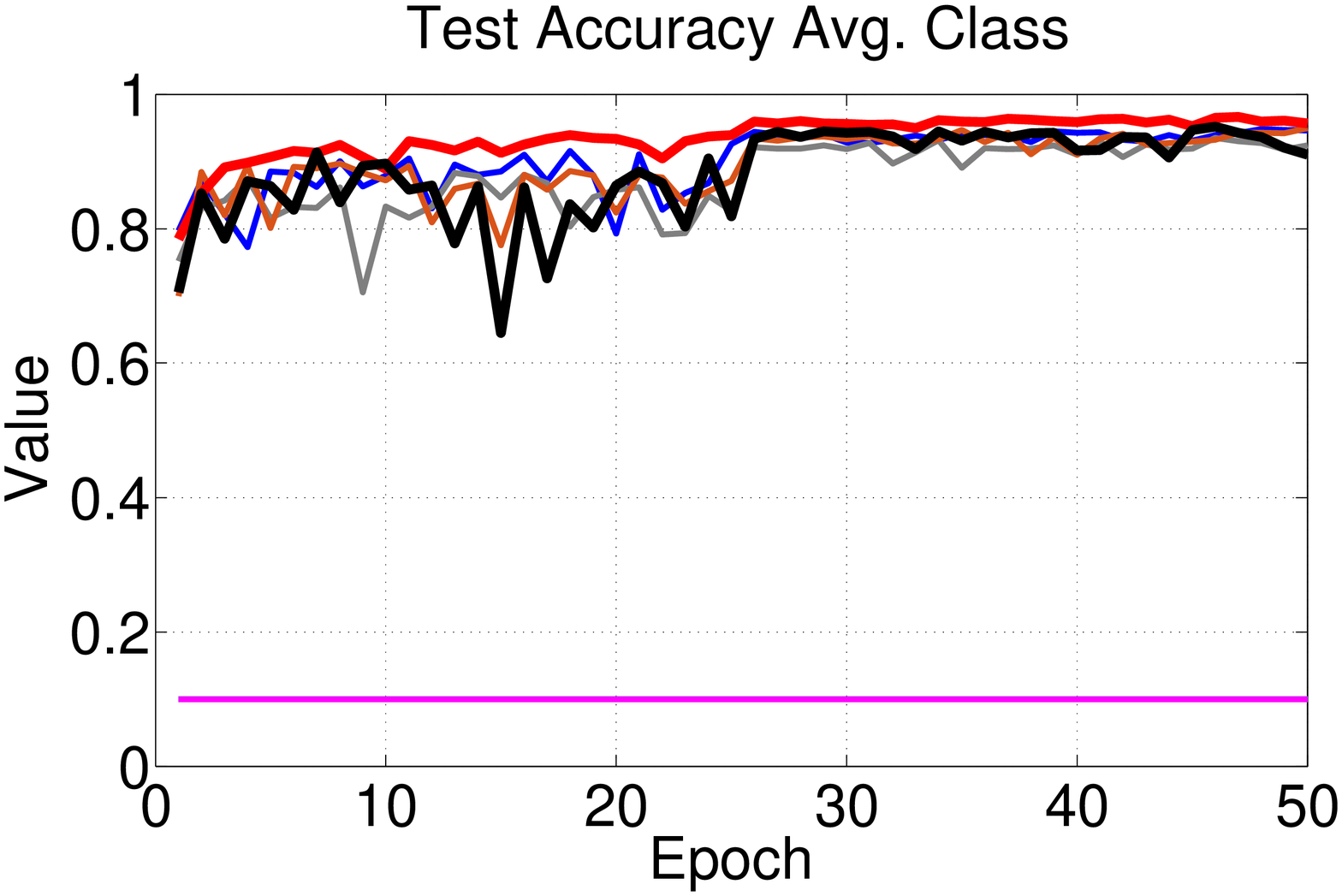}}\hfill
    \subfigure{\includegraphics[width=0.25\linewidth]{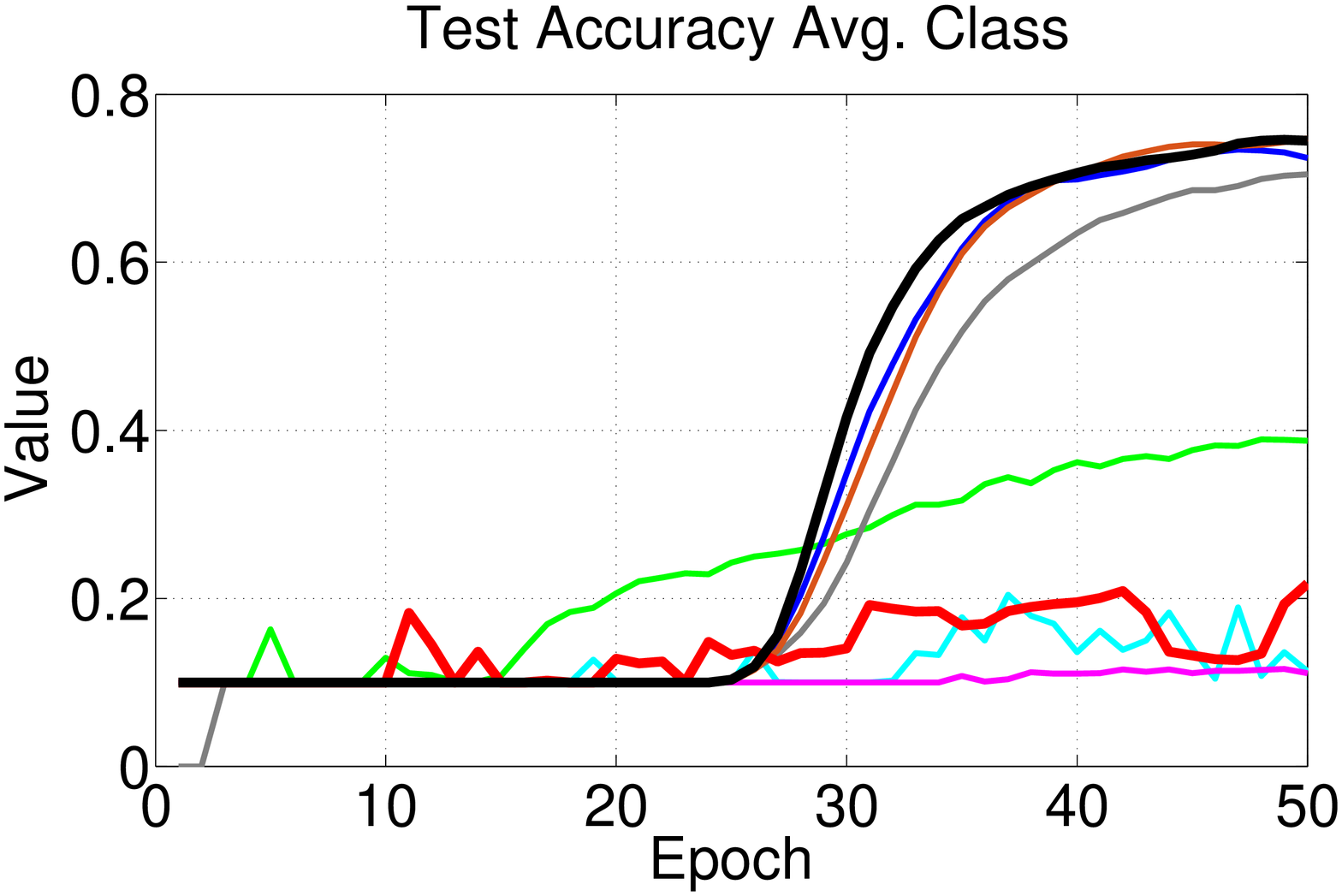}}\hfill
    \subfigure{\includegraphics[width=0.25\linewidth]{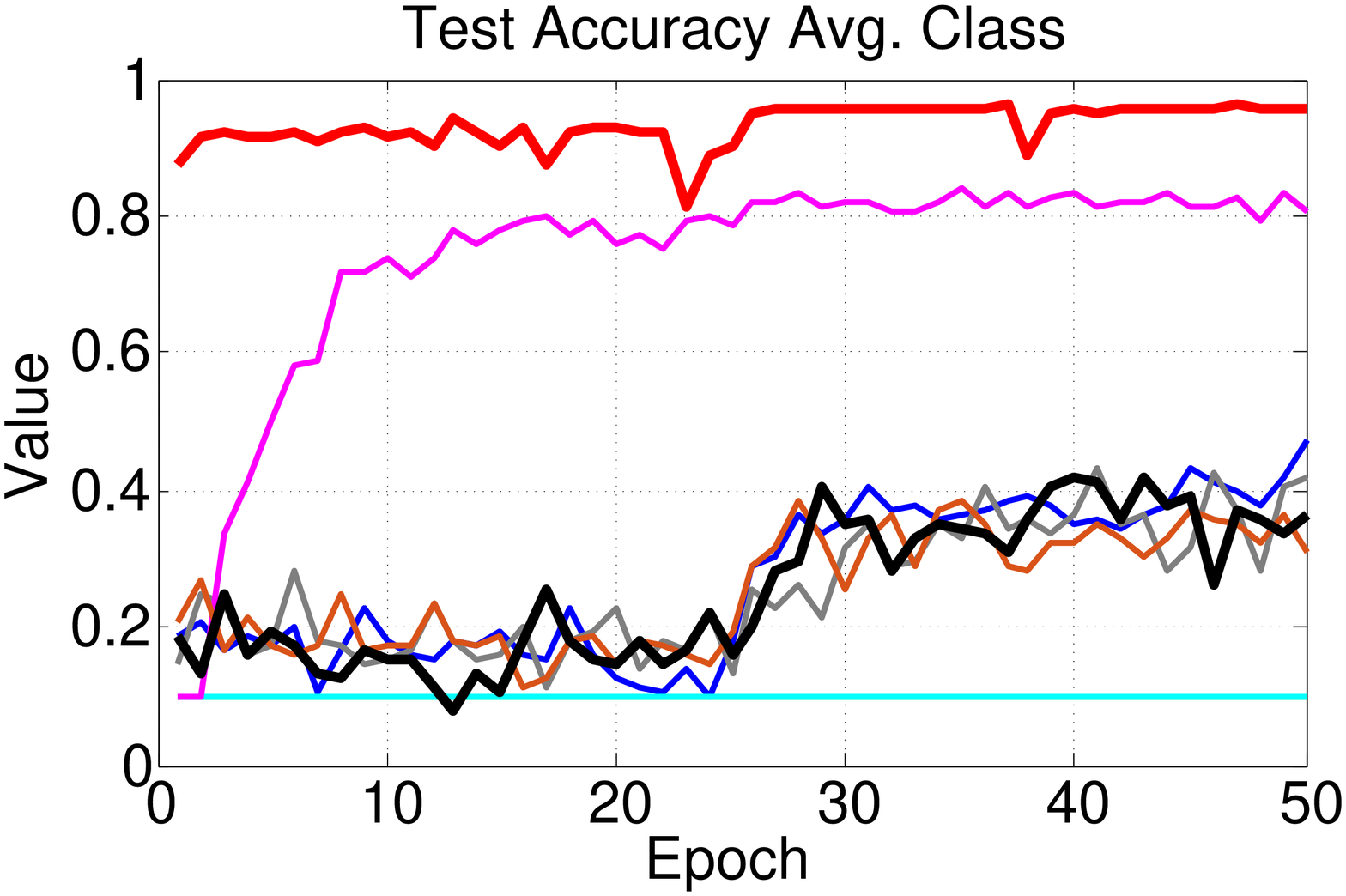}}\hfill
    \caption{Result comparison on MNIST: {\bf (left->right)} default setting, $lr=0.01$, $ts=200$, and $bs=2$.}
	\label{fig:4_cases_mnist}
\end{figure*}

\begin{table*}[t]
	\centering\footnotesize
	\setlength\tabcolsep{2pt}
	\begin{tabular}{|c||c|c||c|c||c|c||c|c||c|c|}
		\hline
		\multirow{2}{*}{*} & \multicolumn{2}{c||}{default setting}  & \multicolumn{2}{c||}{$lr=0.01$} & \multicolumn{2}{c||}{$ts=440$} & \multicolumn{2}{c||}{$bs=2$} & \multicolumn{2}{c|}{ave. of 4 settings} \\ \hline
		{} & avg. cls & overall & avg. cls & overall& avg. cls & overall & avg. cls & overall & avg. cls & overall \\ \hline
		NT           & 0.800 & 0.840  & 0.025  & 0.041   & 0.336  & 0.416 & 0.025 & 0.041 & 0.297 & 0.335 \\ \hline   
		SRIP-v1+NT  & 0.815 & 0.849  & 0.025  & 0.041   & 0.428  & 0.495 & 0.025 & 0.041 & 0.340 & 0.357  \\ \hline
		SRIP-v2+NT & {\bf 0.822} & {\bf 0.858}  & 0.025  & 0.041   & 0.290  & 0.371 & 0.025 & 0.041 & 0.309 & 0.328 \\ \hline 
		{\bf Ours+NT}         & 0.814 & 0.850  & {\bf 0.712}  & {\bf 0.770}   & {\bf 0.520}  & {\bf 0.571} & {\bf 0.793} & {\bf 0.829} & {\bf 0.732} & {\bf 0.755} \\ \hline \hline
		CT             & 0.819 & 0.860  & 0.767  & 0.818   & 0.591  & 0.666 & {\bf 0.242} & {\bf 0.369} & {\bf 0.642} & {\bf 0.678} \\ \hline
		SRIP-v1+CT  & 0.812 & 0.856  & 0.719  & 0.780   & 0.598  & 0.673 & 0.088 & 0.143 & 0.554 & 0.613  \\ \hline
		SRIP-v2+CT & {\bf 0.824} & {\bf 0.863}  & {\bf 0.787}  & {\bf 0.839}   & 0.612  & 0.682 & 0.138 & 0.223 & 0.590 & 0.652 \\ \hline 
		{\bf Ours+CT}               & 0.821 & 0.862 & 0.775   & 0.832   & {\bf 0.638}  & {\bf 0.685} & 0.132 & 0.214 & 0.620 & 0.648 \\ \hline
	\end{tabular}
	\caption{Best test accuracy comparison on ModelNet40.}
	\label{tab:3D_classification_table}
\end{table*}

\begin{table*}[t]
 	\centering\footnotesize
 	\setlength\tabcolsep{2pt}
 	\begin{tabular}{|c||c|c||c|c||c|c||c|c||c|c|}
 		\hline
 		\multirow{2}{*}{}& \multicolumn{2}{c||}{default setting}  & \multicolumn{2}{c||}{$lr=0.01$} & \multicolumn{2}{c||}{$ts=200$} & \multicolumn{2}{c||}{$bs=2$} & \multicolumn{2}{c|}{ave. of 4 settings} \\ \hline
 		{} & avg. cls & overall & avg. cls & overall& avg. cls & overall & avg. cls & overall & avg. cls & overall \\ \hline
 		NT           & 0.977 & 0.977  & 0.967  & 0.967   & \textbf{0.398}  & \textbf{0.389} & 0.114 & 0.100 & 0.614 & 0.608 \\ \hline   
 		SRIP-v1+NT  & 0.934 & 0.934  & 0.934  & 0.933   & 0.204  & 0.204 & 0.114 & 0.100 & 0.547 & 0.543  \\ \hline
 		SRIP-v2+NT & 0.928 & 0.927  & 0.928  & 0.927   & 0.119  & 0.116 & 0.928 & 0.927 & 0.726 & 0.724 \\ \hline 
 		{\bf Ours+NT}         & \textbf{0.978} & \textbf{0.978}  & \textbf{0.978}  & \textbf{0.978}   & 0.223  & 0.217 & \textbf{0.963} & \textbf{0.963} & {\bf 0.786} & \textbf{0.784} \\ \hline \hline
 		CT  &  0.967 & 0.967  & 0.976  & 0.976   & 0.741  & 0.734 & \textbf{0.481} & \textbf{0.470} & {\bf 0.791} & {\bf 0.787} \\ \hline
 		SRIP-v1+CT  & 0.975 & 0.975  & 0.975  & 0.975   & 0.713  & 0.705 & 0.222 & 0.206 & 0.721 & 0.715  \\ \hline
 		SRIP-v2+CT & 0.974 & 0.973  & 0.974 & 0.973   & 0.751  & \textbf{0.748} & 0.181 & 0.166 & 0.720 & 0.715 \\ \hline 
 		{\bf Ours+CT}               & \textbf{0.976} & \textbf{0.976} & \textbf{0.977}   & \textbf{0.977}   & {\bf 0.752}  & 0.746 & 0.433 & 0.421 & 0.785 & 0.780 \\ \hline
 	\end{tabular}
 	\caption{Best test accuracy comparison on 2D point clouds of MNIST.}
 	\label{tab:2D_mnist_point_cloud_table}
 	\vspace{-3mm}
 \end{table*}
 
\begin{figure*}[t]
	\begin{minipage}[b]{0.245\linewidth}
		\begin{center}
			\centerline{\includegraphics[clip=true,width=1.1\linewidth]{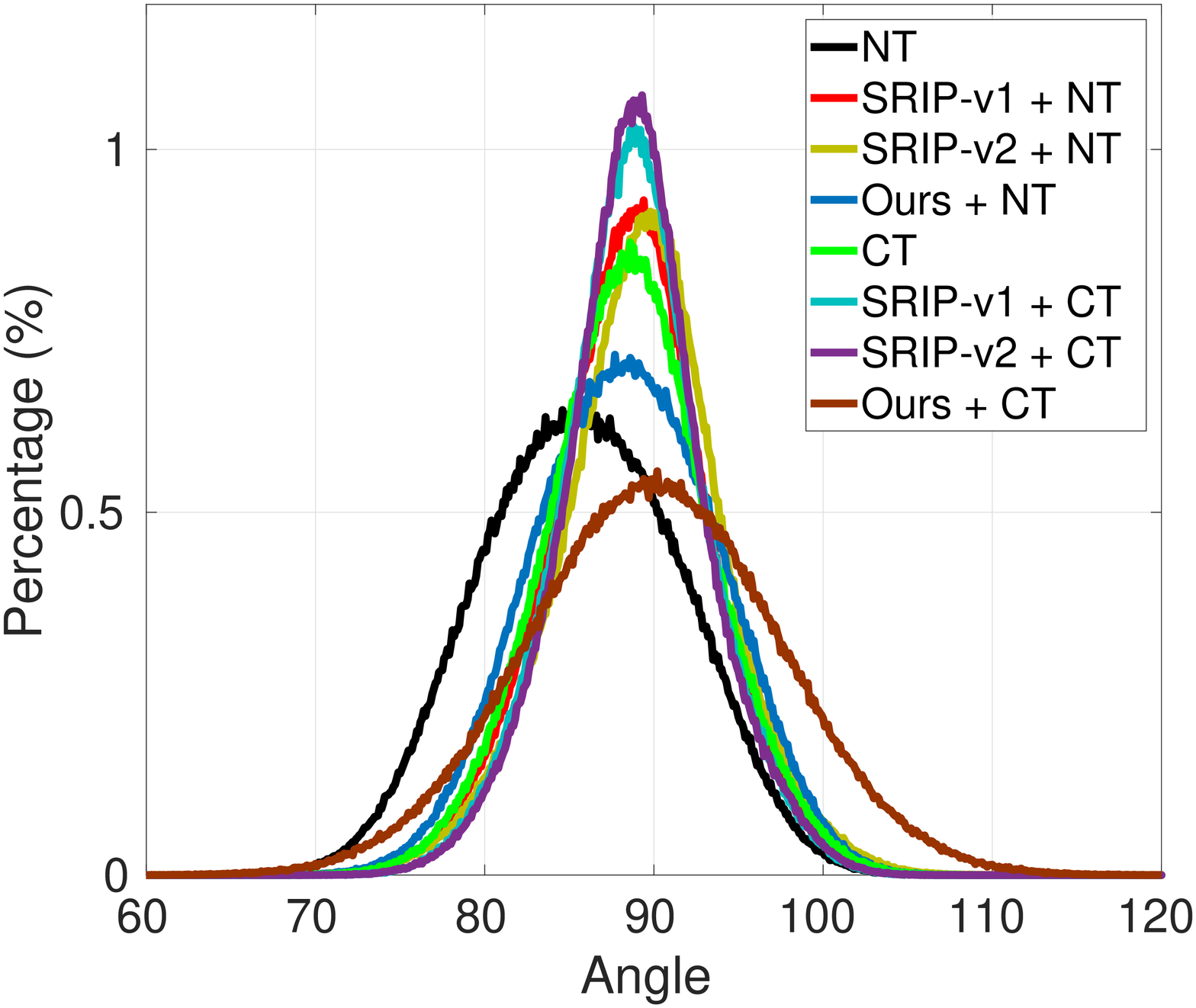}}	
		\end{center}
	\end{minipage}
	\begin{minipage}[b]{0.245\linewidth}
		\begin{center}
			\centerline{\includegraphics[clip=true,width=1.1\linewidth]{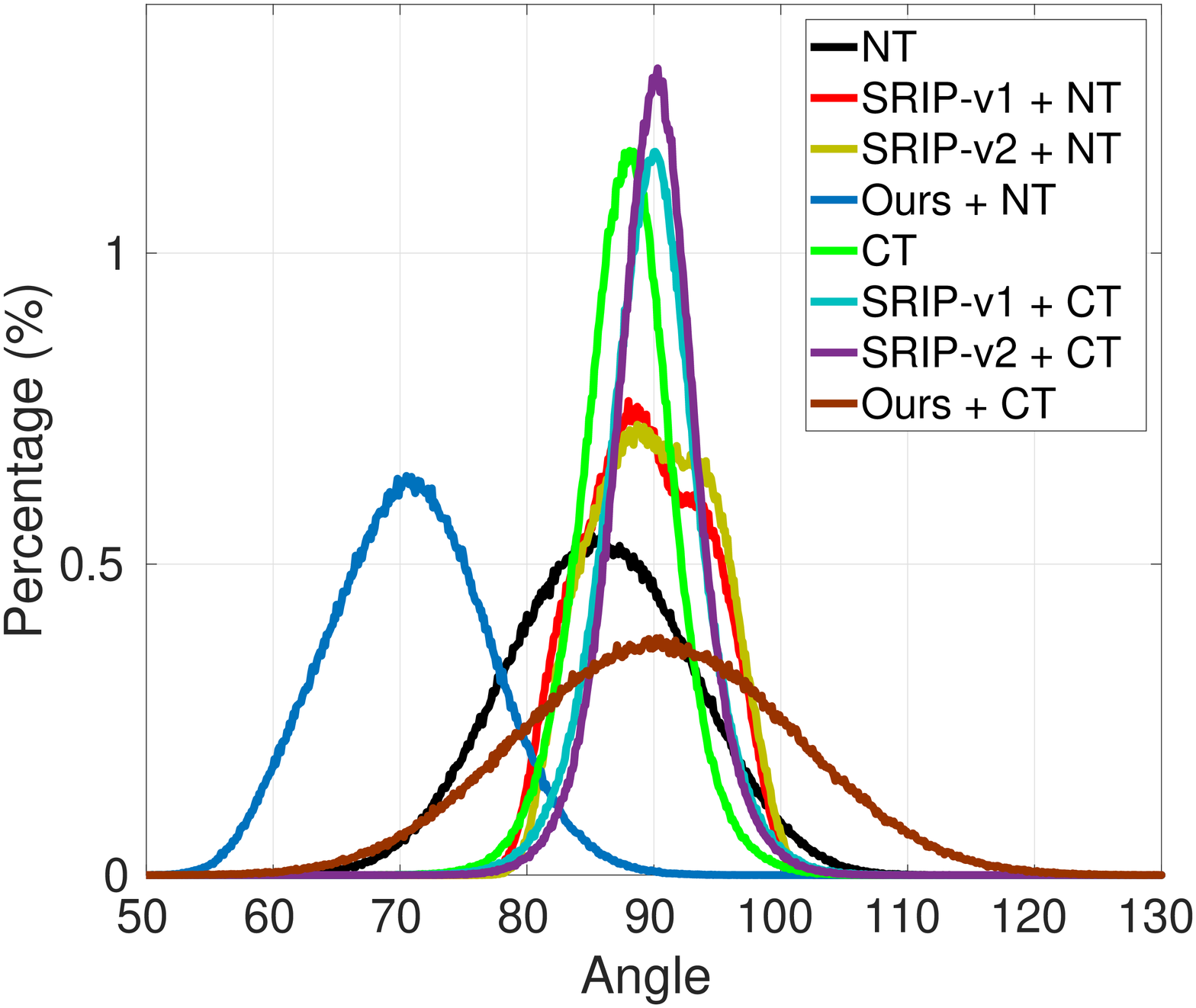}}
		\end{center}
	\end{minipage}	
	\begin{minipage}[b]{0.245\linewidth}
		\begin{center}
			\centerline{\includegraphics[clip=true,width=1.1\linewidth]{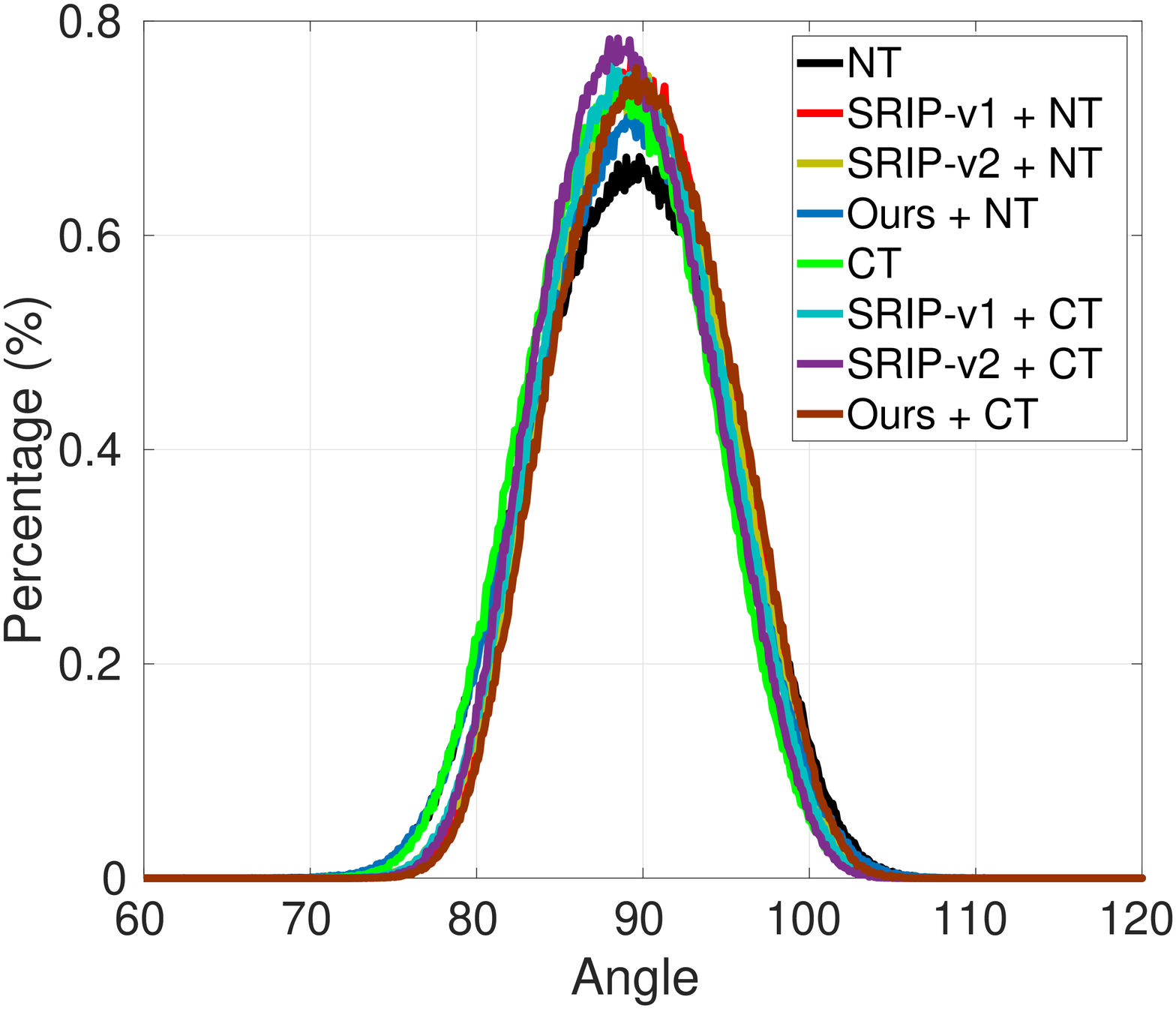}}
		\end{center}
	\end{minipage}
	\begin{minipage}[b]{0.245\linewidth}
		\begin{center}
			\centerline{\includegraphics[clip=true,width=1.1\linewidth]{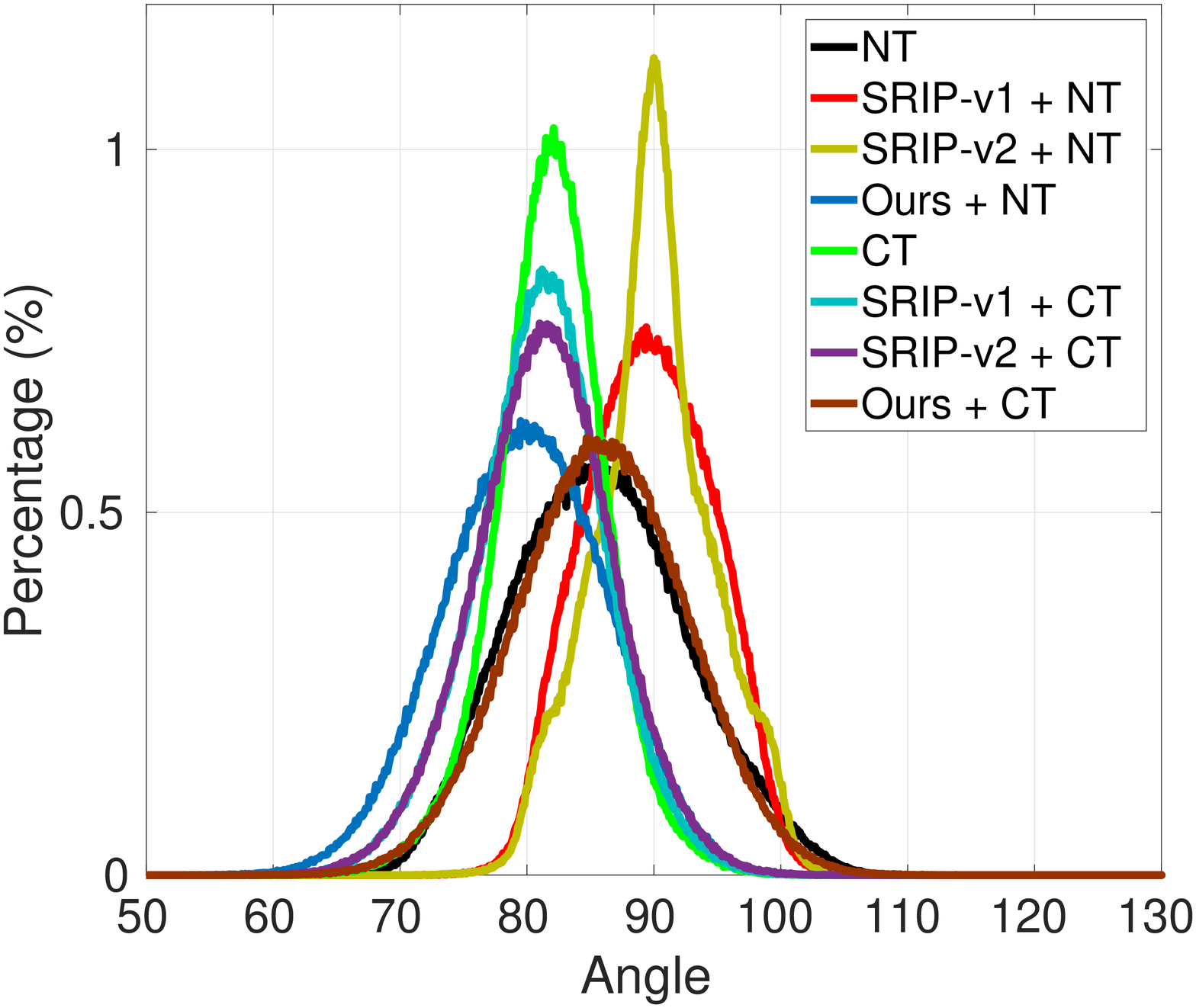}}	
		\end{center}
	\end{minipage}
	\vspace{-5mm}		
	\caption{The learned angular distributions on ModelNet40: {\bf (left->right)} default setting, $lr=0.01$, $ts=440$, and $bs=2$.}
	\label{fig:learned_angular}
	\vspace{-3mm}
\end{figure*}

\section{Experiments}\label{sec:exp}
We study the impact of OR on deep learning using the task of  point cloud classification.

{\bf Learning Scenarios:}
In order to identify the gain of OR, we intentionally design some extreme learning scenarios where we only change one hyper-parameter from the default setting while keeping the rest unchanged, 
listed as follows:
\begin{itemize}[leftmargin=4mm, nosep]
    \item {\em Default Setting:} Under this setting, we train each model using the full training data set and then test it on the full testing data set. We keep all the hyper-parameters unchanged in the public code. 
    \item {\em Large Initial Learning Rate ($lr$):} Here we only change the initial learning rate to 0.01.
    \item {\em Limited Training Samples ($ts$):} Here we only use small number of training samples uniformly selected from the entire training set at random. 
    \item {\em Small Batch Size ($bs$):} Here we only change the batch size to 2. 
\end{itemize}

{\bf Baselines:}
In order to do comparison fairly, we consider the OR algorithms whose code is publicly available and can run in our task. In this sense, we finally have the following baseline algorithms\footnote{We fail to integrate OrthoReg \cite{rodriguez2016regularizing} with our point cloud classification task, neither using their code nor reimplementing it, because we realize that it is so difficult to modify SGD to achieve reasonable performance or even make it work.}:

\begin{itemize}[leftmargin=4mm, nosep]
    \item {\em Naive Training (NT):} In this baseline, we train a network without any regularization.
    \item {\em Conventional Training (CT):} In this baseline, we train a network with weight decay (5e-4), dropout (keep-ratio 0.7), and batch normalization (BN).
    \item {\em Spectral Restricted Isometry Property~\cite{Bansal2018nips} (SRIP-v1 \& SRIP-v2)\footnote{From~\cite{srip_github_link},  
    we find there are two implementations of the approach, under the folders ``Wide-Resnet'' and ``SVHN'', respectively. We therefore name them as ``v1'' and ``v2''.}:} This regularizer is defined as $\mathcal{R} = \lambda\cdot\sigma\left(\mathbf{W}^T\mathbf{W}-\mathbf{I}\right)$,
    where $\mathbf{W}$ is the network weight matrix, $\mathbf{I}$ is an identity matrix, $\sigma(\cdot)$ is the spectral norm, and $\lambda$ is a constant. Currently this is the state-of-the-art OR in the literature.
\end{itemize}

In the sequel we will present our results for different types of networks on different data sets.

\subsection{Multilayer Perceptron (MLP): PointNet}\label{ssec:pointnet}

{\bf Data Sets:}
We conduct comparison on two benchmark data sets, 
ModelNet40~\cite{wu20153d} and 2D point clouds of MNIST~\cite{lecun1998gradient}. ModelNet40 has 12,311 CAD models for 40 object categories, split into 9,840 for training, 2,468 for testing. We uniformly sample 1024 points from meshes to obtain 3D point clouds, and normalize them into a unit ball. MNIST is a handwritten digit data set, consisting of 60,000 training images and 10,000 testing images with $28\times28=784$ gray pixels. We convert each image to 2D point clouds by taking image coordinates of all non-zero pixels. 

{\bf Networks:}
We choose PointNet-vanilla~\cite{Qi_2017_CVPR} (\ie an MLP(64, 64, 64, 128, 1024, 512, 256, 40) without T-nets) as the backbone network and integrate different regularization techniques into it to verify the effects of OR. During training we conduct data augmentation on-the-fly by randomly rotating the points along the up-axis and jittering the coordinates of each point by a Gaussian noise with zero mean and 0.01 std. We use the PyTorch code~\cite{pointnet_github_link}
as our testbed and the same training and evaluation protocols as the original PointNet code. The optimizer is Adam, initial learning rate is 0.001, and momentum is 0.9. The decay rate for batch normalization starts with 0.5, and is gradually increased to 0.99, dropout is set at a ratio of 0.7 until the last fully connected layer. The batch size on ModelNet40 is 32, and on MNIST is 100. We tune each approach to report the best performance.

In order to demonstrate that our findings are common across different MLP, on MNIST we slightly modify the architecture of original PointNet to another MLP(64, 64, 64, 128, 512, 256, 128, 10).

{\bf Training Stability:}
We summarize the training and testing behaviors of each approach on ModelNet40 in Fig. \ref{fig:4_cases_modelnet40}. {\bf (a)}. Under the default setting, our regularizer helps the naive training converge much faster than the others which are appreciated, while with the conventional training it seems such effect is neutralized by other regularizations. In testing, the overall behavior of each approach is similar to each other without significant performance gap. {\bf (b)} Under the setting of $lr=0.01$, naive training and both SRIPs fail to work, while our regularizer with conventional training works reasonably well, it converges faster than others in training and achieves the best performance. {\bf (c)} Under the setting of $ts=440$, the performances of different algorithms differ significantly, even though all the competitors work. Our regularizer with conventional training achieves the best performance, converging around 100 epochs which is much faster than the others. Our regularizer outperforms SRIP significantly as well. {\bf (d)} Under the setting of $bs=2$, only our regularization with naive training works, and surprisingly achieves the very good performance similar to that under the default setting. For conventional training, the regularization dominates the training behavior so strongly that even our regularizer cannot make it work. This is expected as usually conventional regularization techniques cannot work well using very small batch size. In Fig. \ref{fig:4_cases_mnist} we summarize the training and testing behaviors of each approach on MNIST. Similar observations can be made here in both training and testing. 

In summary, for MLP we do not observe any significant gain on accuracy using SRIP, statistically speaking on average, but some improvement on convergence in training under the default setting. Our self-regularization, however, improves both naive training and conventional training significantly on both convergence and accuracy, indicating better training stability from our method.

{\bf Accuracy:}
We summarize in Table \ref{tab:3D_classification_table} and Table \ref{tab:2D_mnist_point_cloud_table} the best accuracy of each approach per setting in Fig. \ref{fig:4_cases_modelnet40} and Fig. \ref{fig:4_cases_mnist}, respectively. Compared with other regularization techniques, we conclude that our regularizer can work well not only under well-defined setting but also under extreme learning cases. For instance, the naive training with our self-regularization works as well as, or even better than, the conventional training. These observations imply our regularizer has much better generalization ability for optimizing deep networks, potentially leading to broader applications.

{\bf Angular Distributions:}
We illustrate the learned angular distributions in Fig. \ref{fig:learned_angular}. First of all, it seems that all the well-trained models under the default setting form Gaussian-like distributions with mean around \ang{90}. The variances of these distributions, however, are larger than those in Fig.~\ref{fig:angle_dist}. We conjecture that the main reason for this is that in PointNet the input dimension is usually smaller than the number of filters per hidden layer so that it is hard to achieve (near) orthogonality among all the filters. Next, our self-regularization with naive training always learns Gaussian-like distributions, while conventional training sometimes has negative effects on our learned distributions such as shifting. This is probably due to BN as it normalizes the statistics in gradients. Finally, it seems that spike-shape distributions tend to produce bad models. This observation has also been made in the angular distributions with random weights.

\begin{figure*}[t]
    \subfigure{\includegraphics[width=0.25\linewidth]{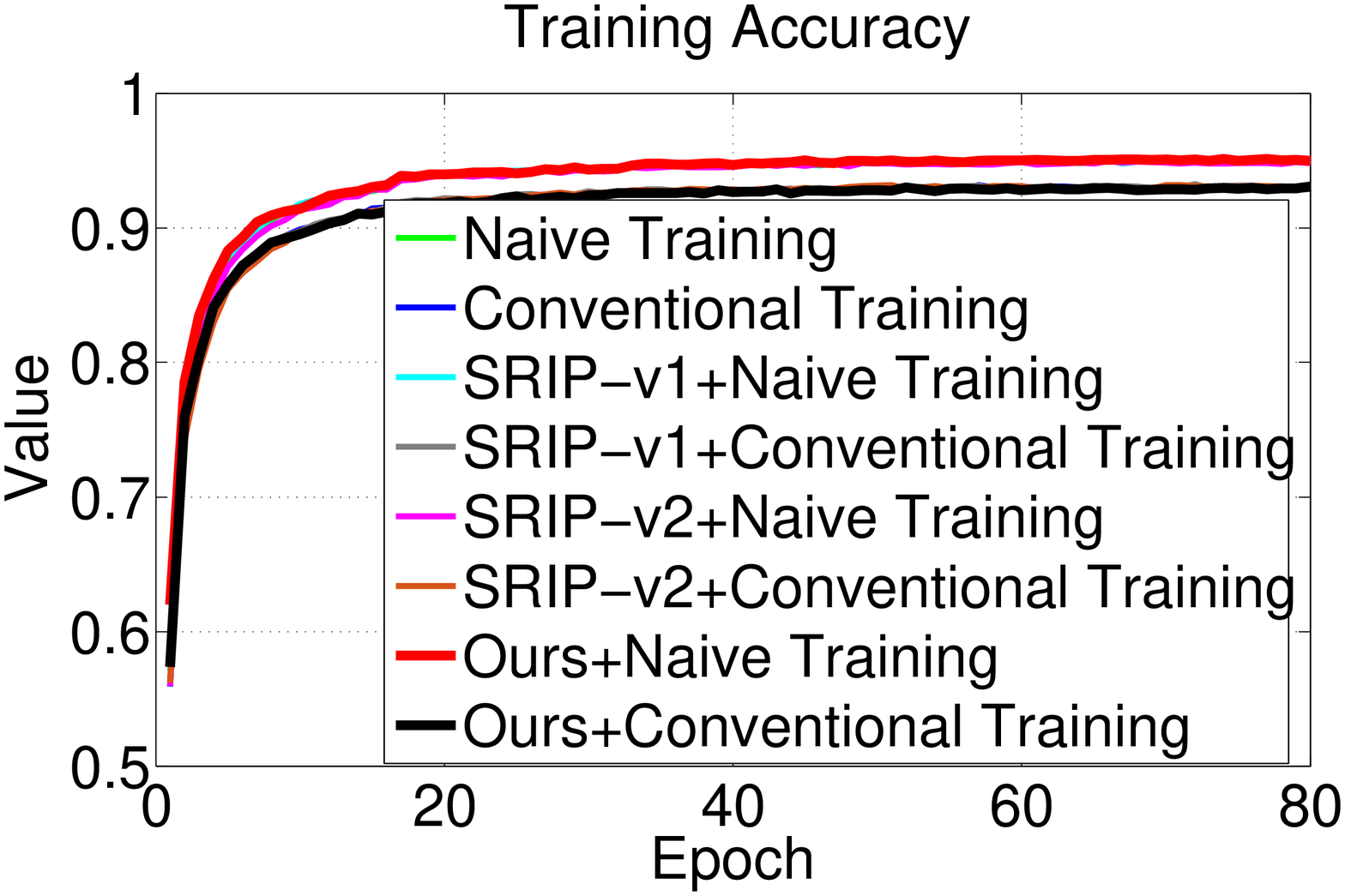}}\hfill
    \subfigure{\includegraphics[width=0.25\linewidth]{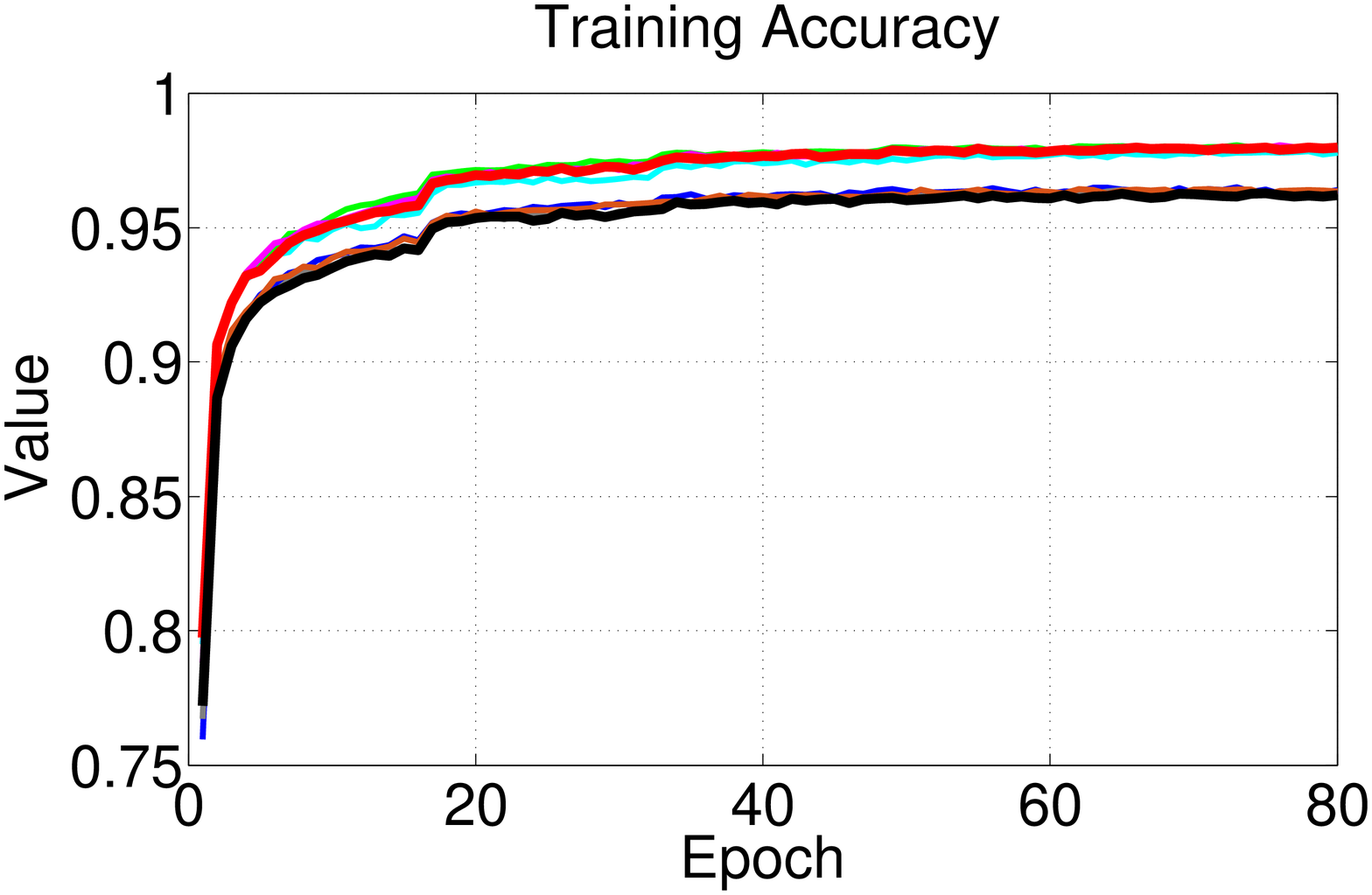}}\hfill
    \subfigure{\includegraphics[width=0.25\linewidth]{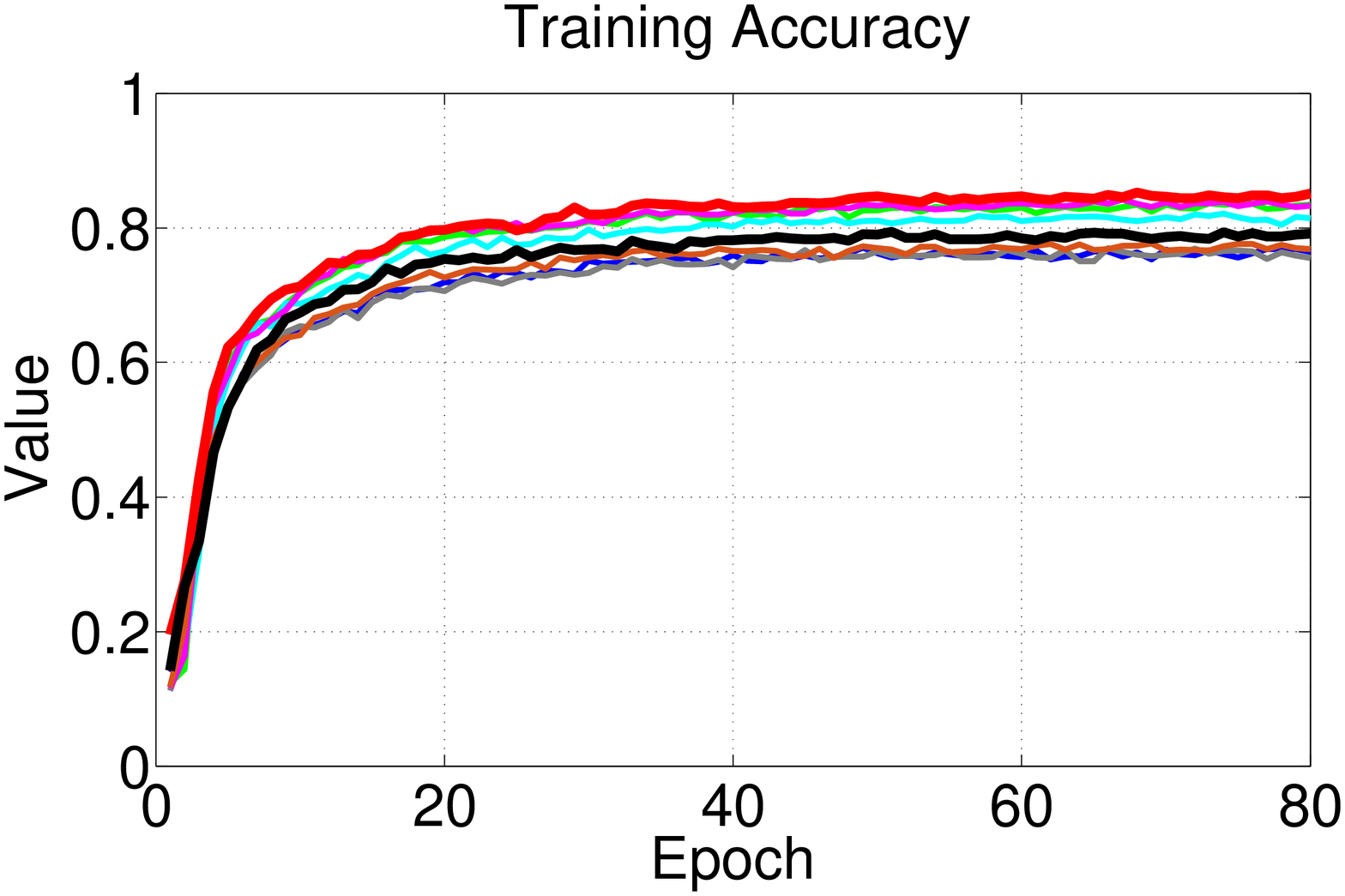}}\hfill
    \subfigure{\includegraphics[width=0.25\linewidth]{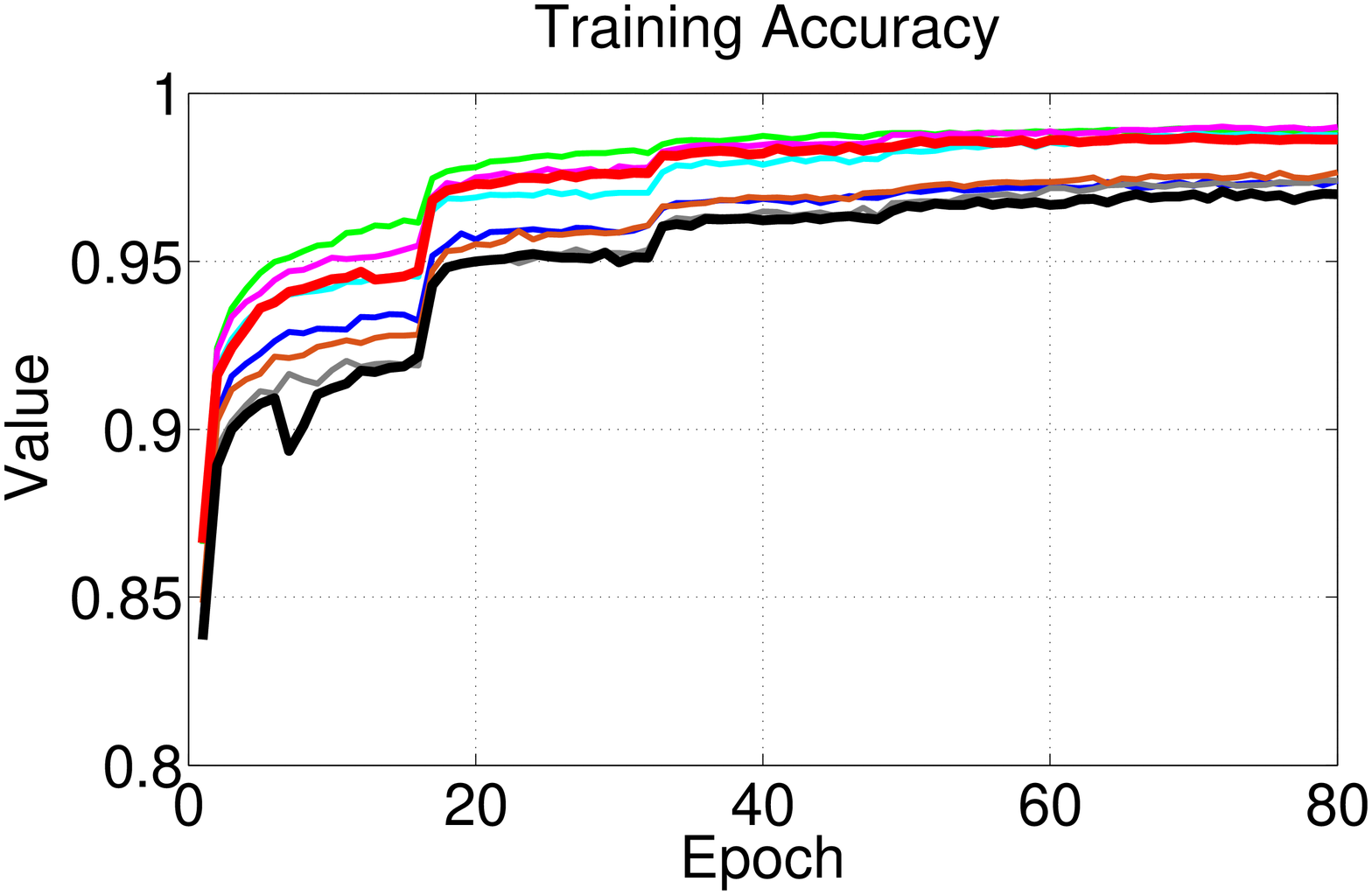}}\hfill
    \subfigure{\includegraphics[width=0.25\linewidth]{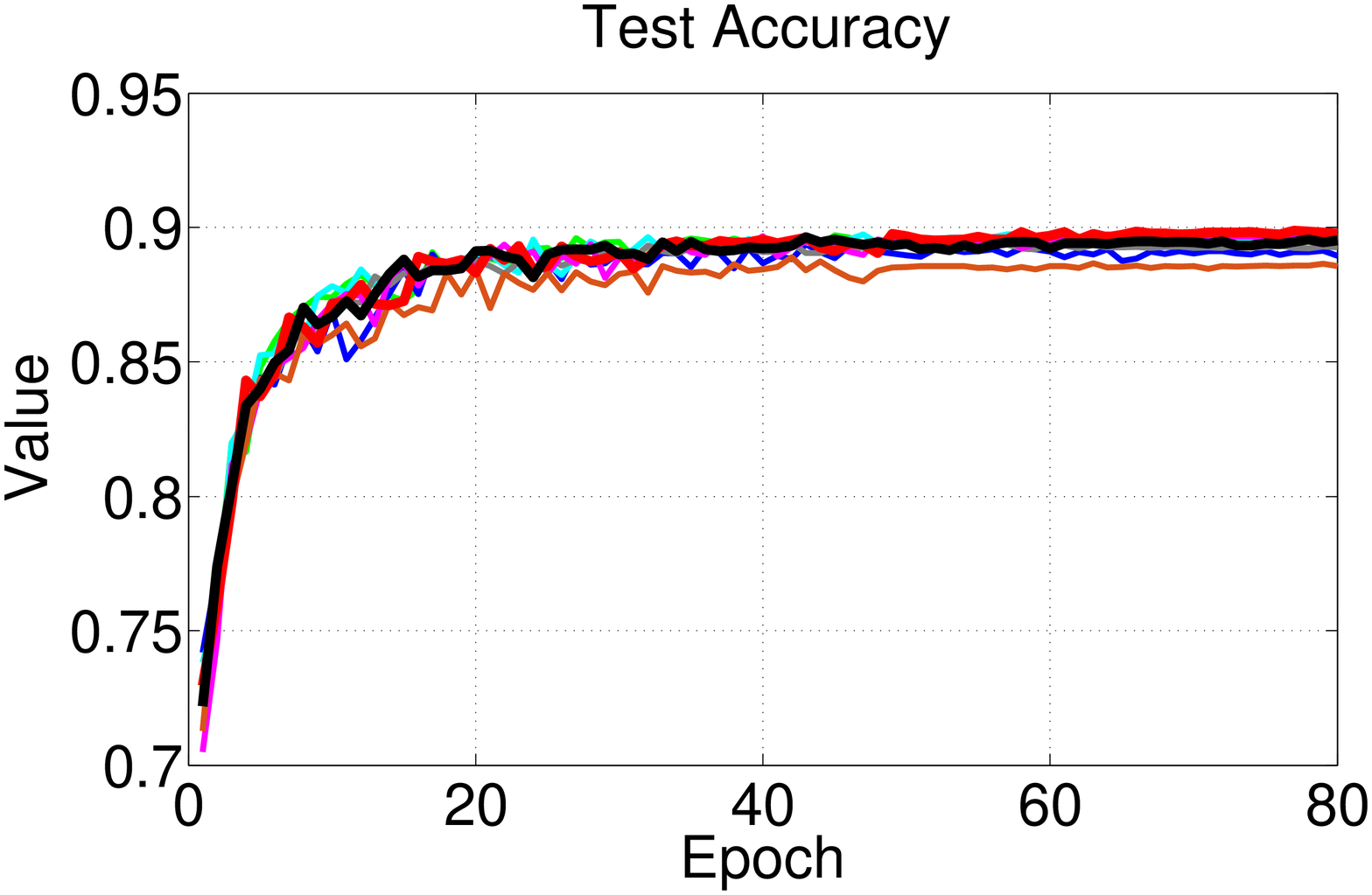}}\hfill
    \subfigure{\includegraphics[width=0.25\linewidth]{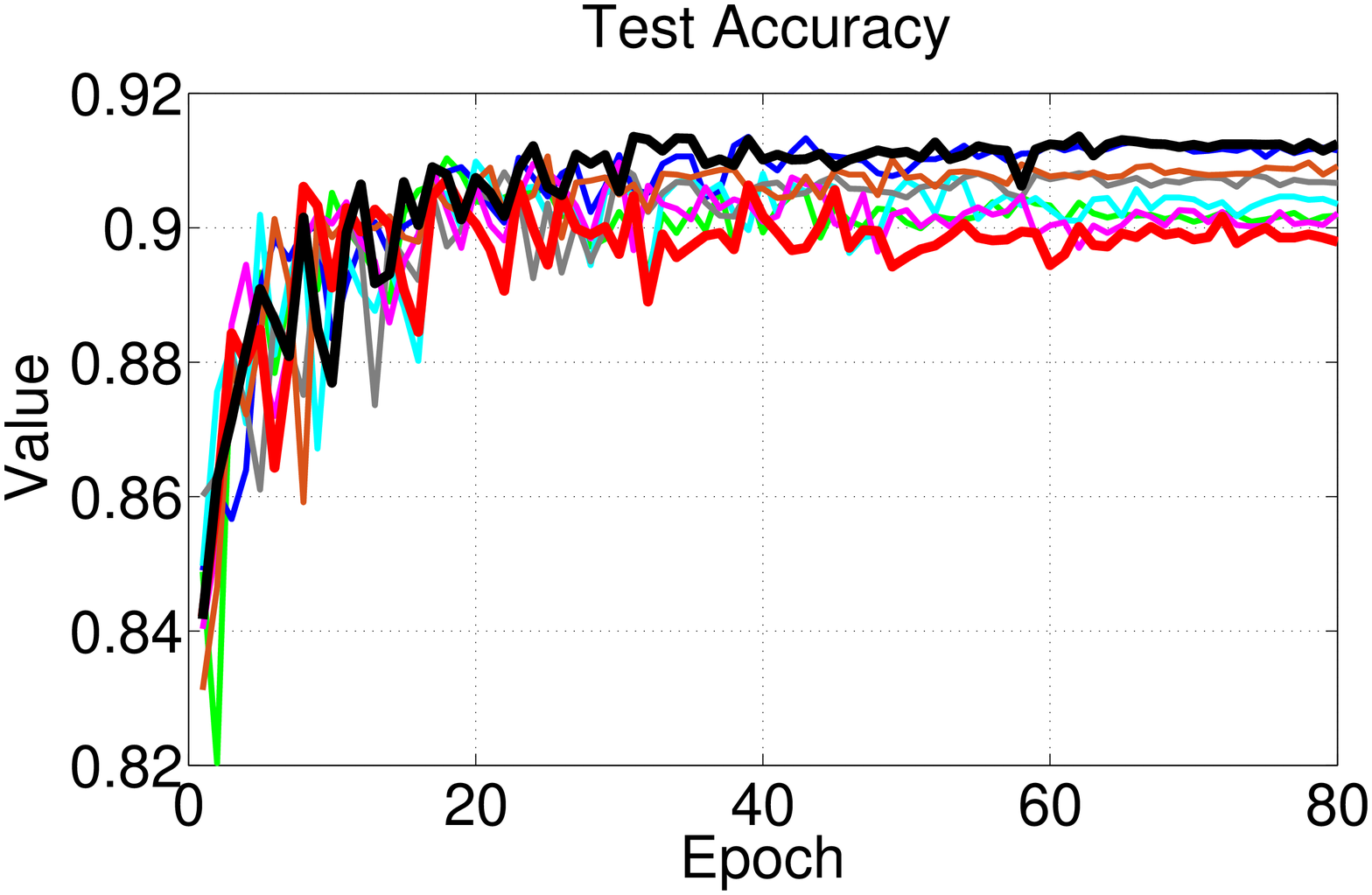}}\hfill
    \subfigure{\includegraphics[width=0.25\linewidth]{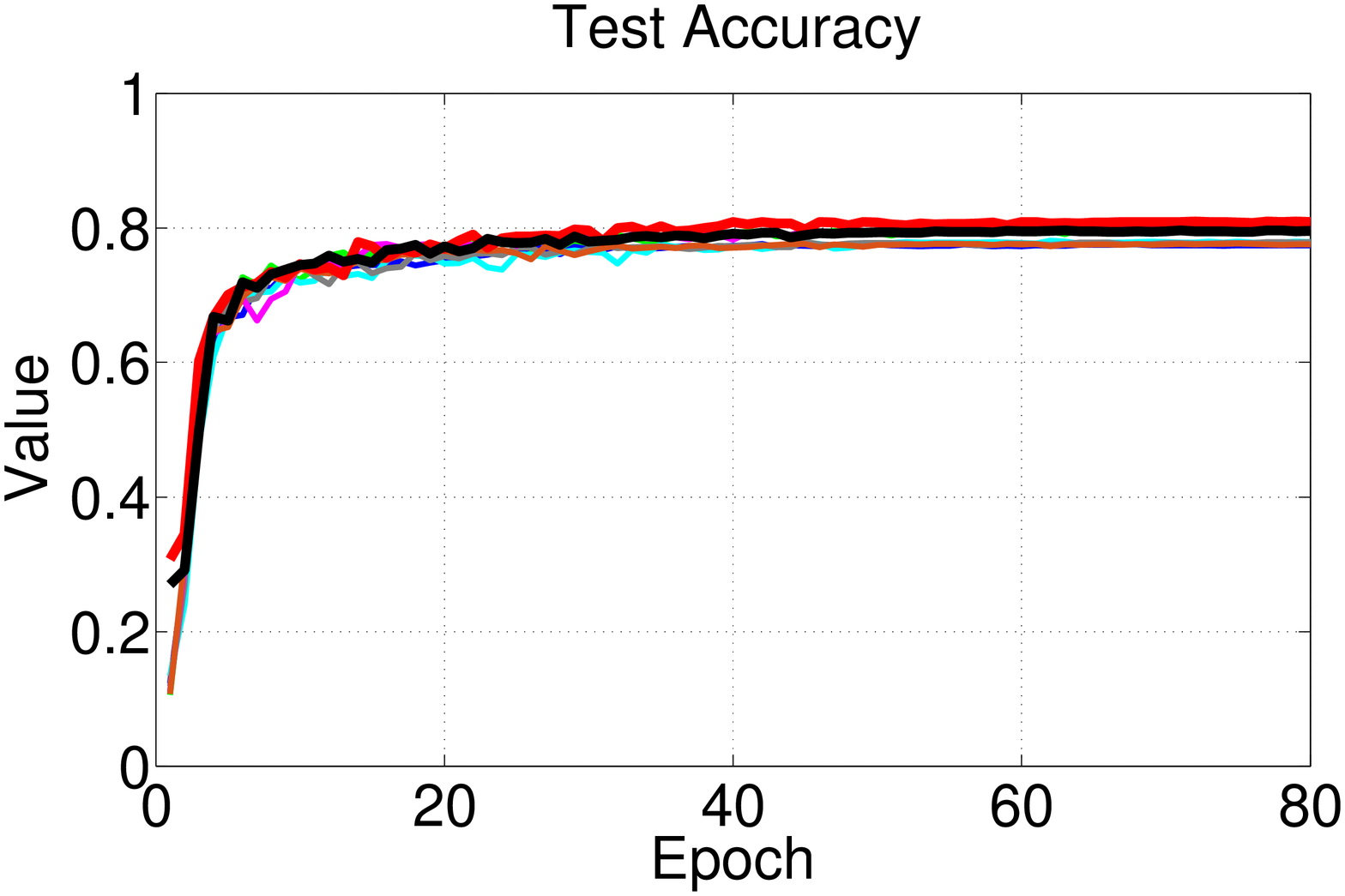}}\hfill
    \subfigure{\includegraphics[width=0.25\linewidth]{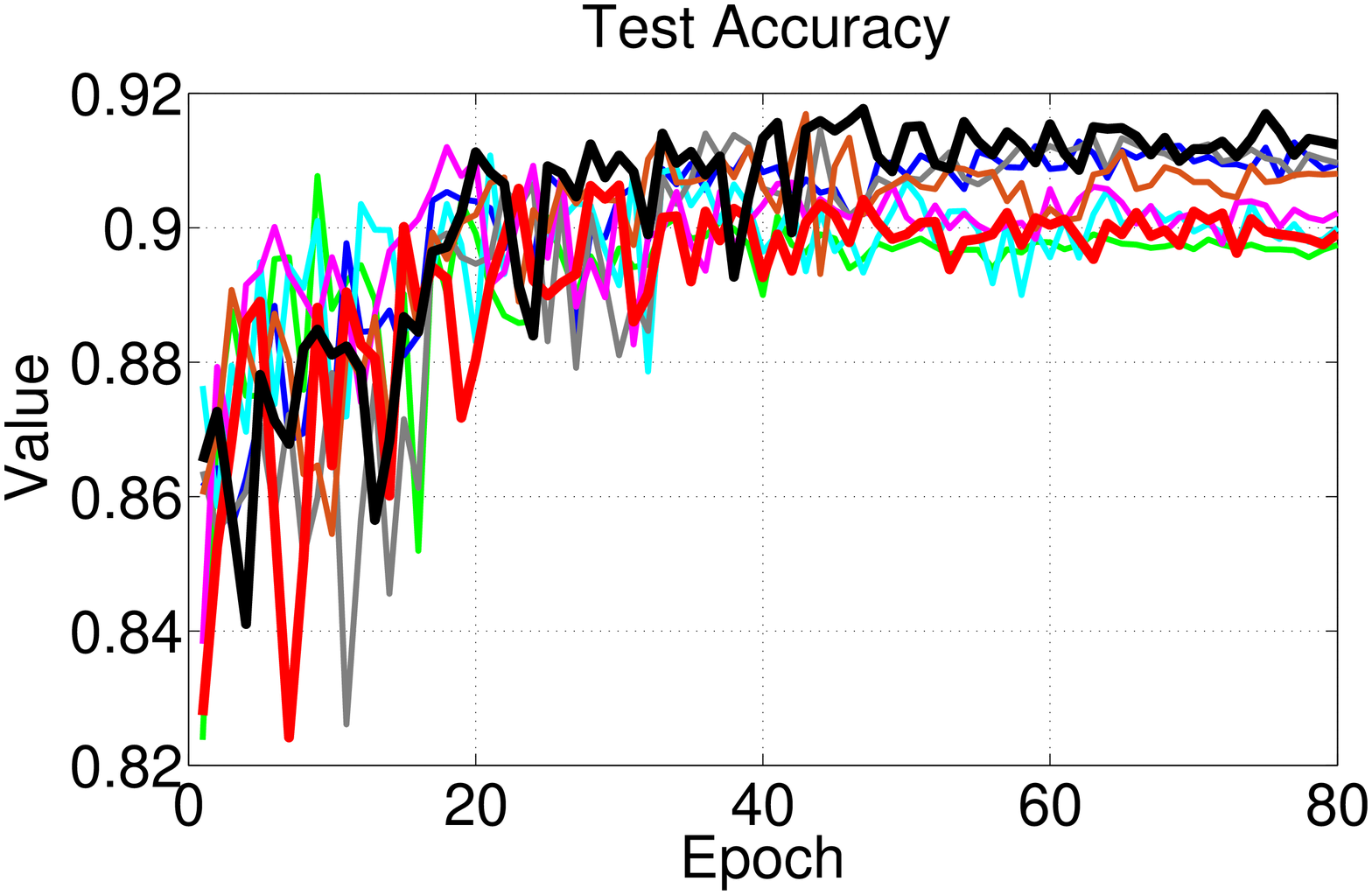}}\hfill
    \caption{Result comparison on ModelNet10: {\bf (left->right)} default setting, $lr=0.01$, $ts=400$, and $bs=2$.}
	\label{fig:ModelNet10_voxnet_4_cases}
 	\vspace{3mm}
\end{figure*}

\begin{figure*}[t]
    \subfigure{\includegraphics[width=0.25\linewidth]{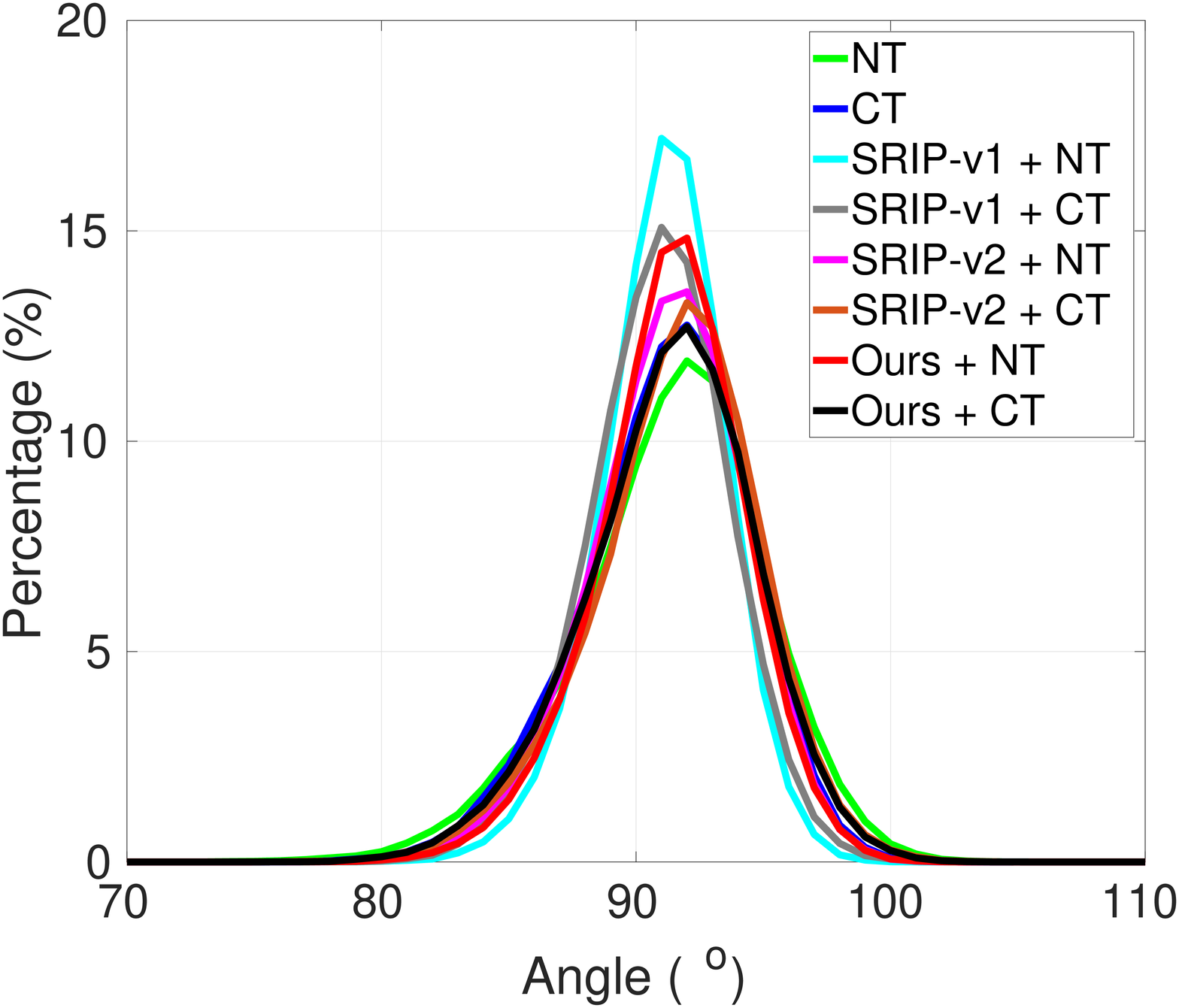}}\hfill
    \subfigure{\includegraphics[width=0.25\linewidth]{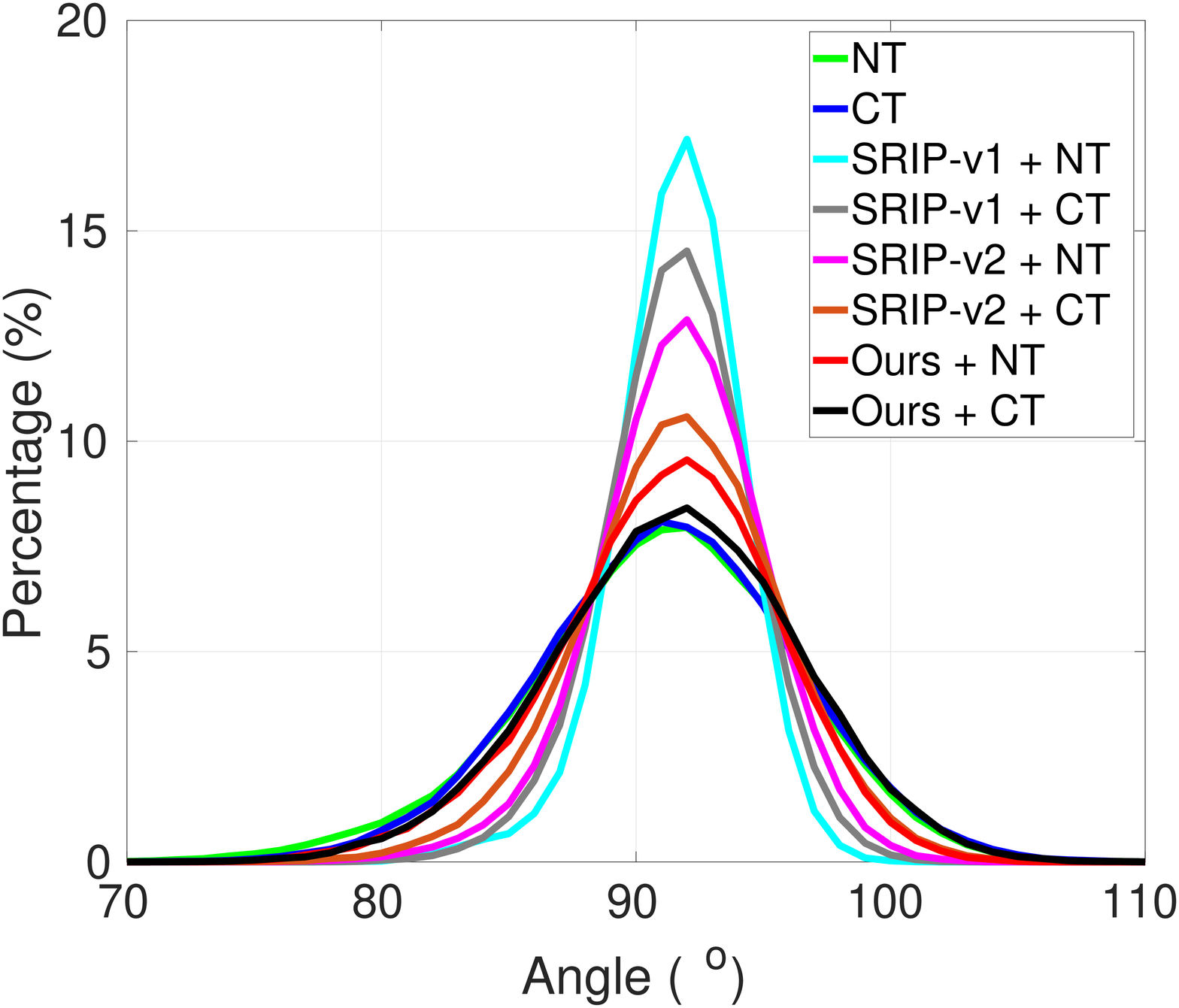}}\hfill
    \subfigure{\includegraphics[width=0.25\linewidth]{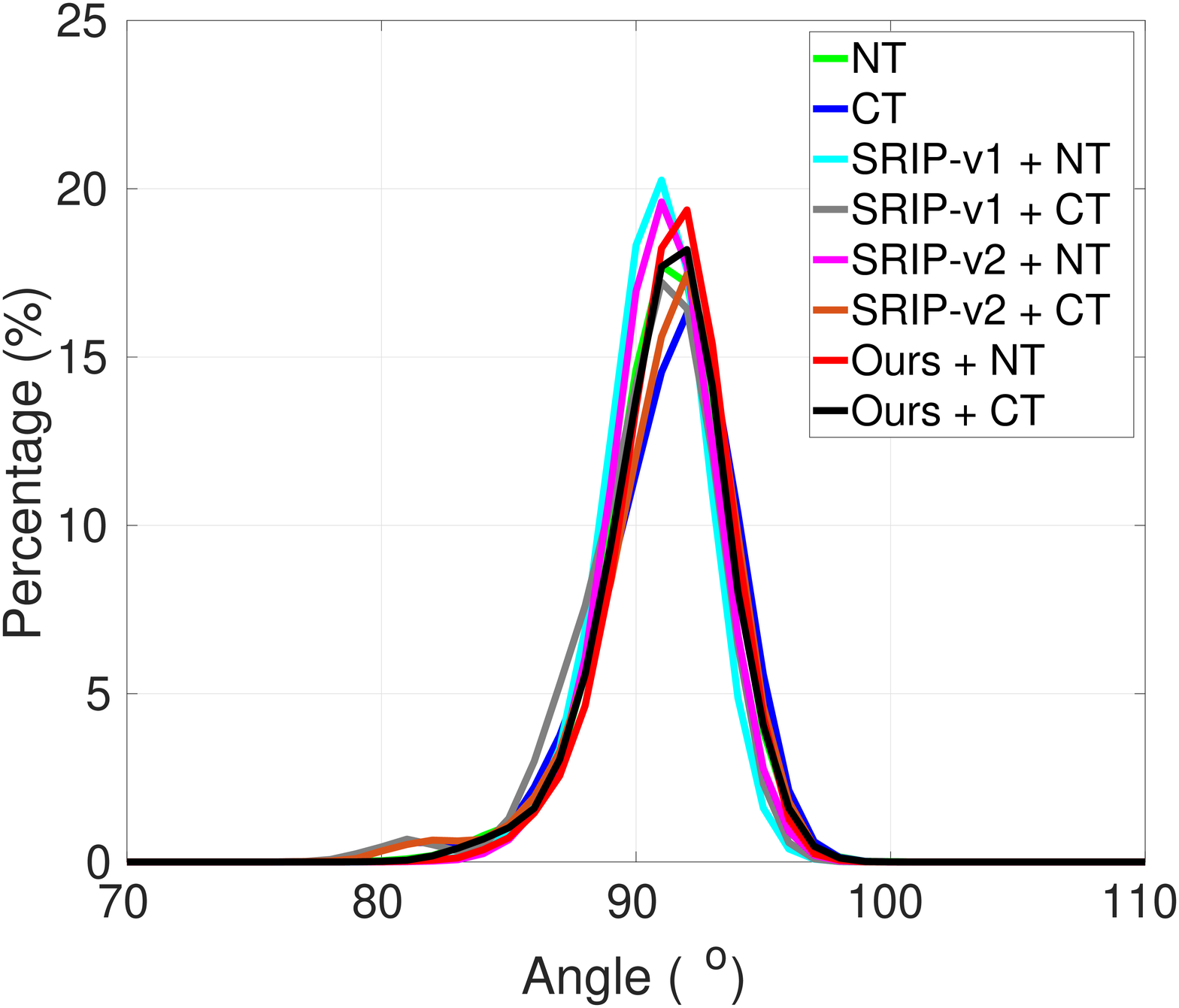}}\hfill
    \subfigure{\includegraphics[width=0.25\linewidth]{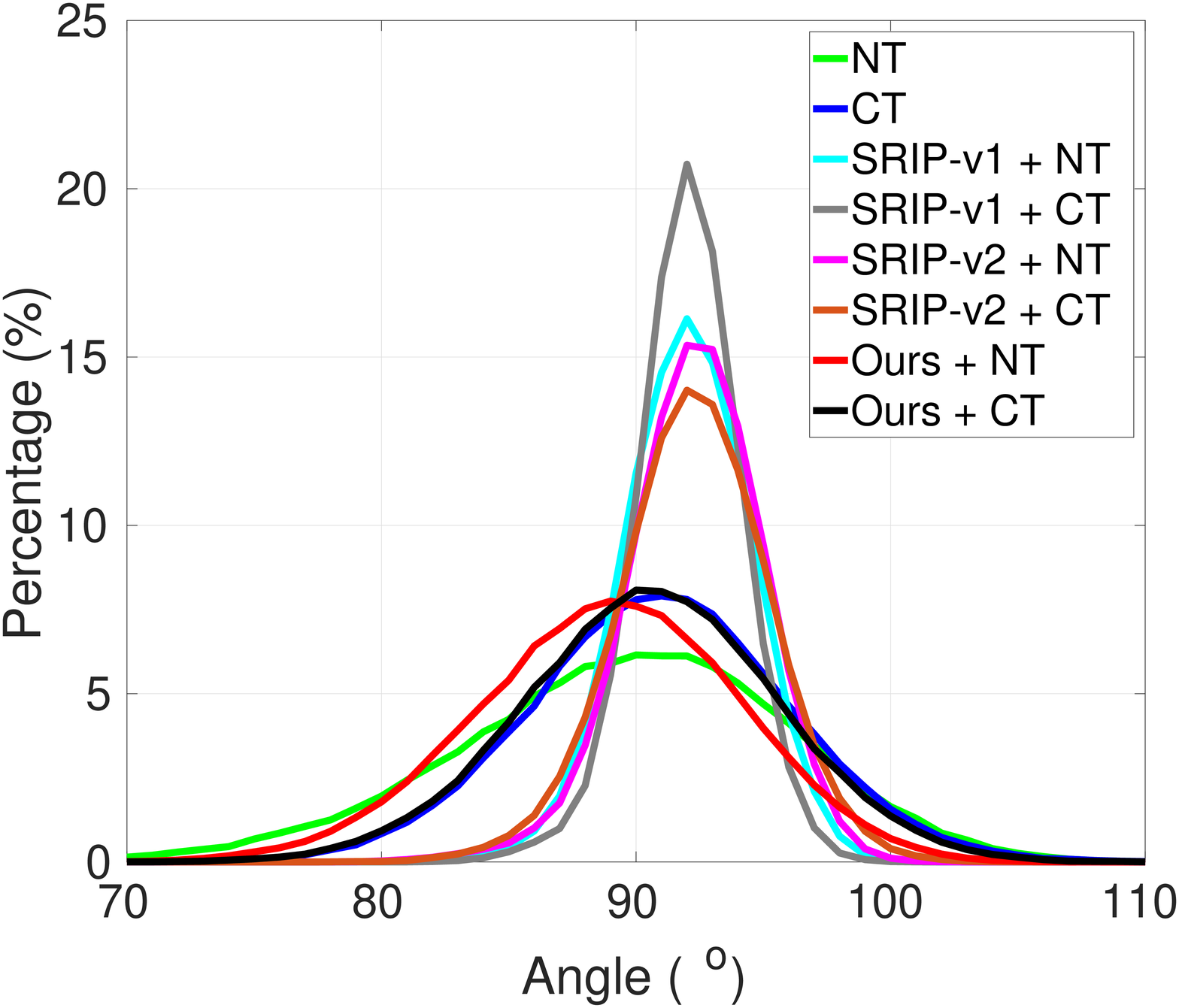}}\hfill
    \caption{The learned angular distributions on ModelNet10: {\bf (left->right)} default setting, $lr=0.01$, $ts=400$, and $bs=2$.}
	\label{fig:ModelNet10_distributions}
	\vspace{3mm}
\end{figure*}

\subsection{Convolutional Neural Networks: VoxNet}\label{ssec:pointnet}

\begin{table}[t]
	\begin{minipage}[b]{\linewidth}
	    \centering\footnotesize
    	\setlength\tabcolsep{2pt}
    	\begin{tabular}{|c|c|c|c|c|c|}
    		\hline
    		& def. s.  & $lr(0.01)$ & $ts(400)$ & $bs(2)$ & ave. \\ \hline		
    		NT           & 0.898 & \textbf{0.910}  & 0.797  & 0.908   & 0.878   \\ \hline   
    		SRIP-v1+NT  & 0.898 & \textbf{0.910}  & 0.781  & 0.911   & 0.875  \\ \hline
    		SRIP-v2+NT & 0.897 & \textbf{0.910}  & 0.800  & \textbf{0.912}   & 0.880  \\ \hline 
    		{\bf Ours+NT}         & \textbf{0.899} & 0.907  & {\bf 0.810}  & 0.906   & {\bf 0.881}   \\ \hline \hline
    		CT             & 0.894 & \textbf{0.914}  & 0.776  & 0.913   & 0.874  \\ \hline
    		SRIP-v1+CT  & 0.894 & 0.909  & 0.779  & 0.915   & 0.874  \\ \hline
    		SRIP-v2+CT & 0.889 & 0.911  & 0.776  & 0.917   & 0.873  \\ \hline 
    		{\bf Ours+CT}               & \textbf{0.896} &\textbf{0.914} & \textbf{0.796}   & \textbf{0.918}   & {\bf 0.881}   \\ \hline
    	\end{tabular}
    	\caption{Best test accuracy comparison on ModelNet10.}
    	\label{tab:ModelNet10_voxnet_4_cases_table}
	\end{minipage}
	\vspace{-20pt}
\end{table}

\begin{table}[t]
	\begin{minipage}[b]{\linewidth}
	    \centering\footnotesize
    	\setlength\tabcolsep{2pt}
        \begin{tabular}{|c||c|c|}
        \hline
                                & lr   & acc. \\ \hline
            NT        & 0.001  & 0.888     \\ \hline   
            SRIP-v1+NT   & 0.001  & 0.893     \\ \hline
            OrthoReg+NT   & -    & -    \\ \hline
            {\bf Ours+NT}   & 0.001  & \bf 0.898    \\ \hline\hline
            CT        & 0.1    & 0.963    \\ \hline
            SRIP-v1+CT       & 0.1    & 0.962    \\ \hline
            OrthoReg+CT   & 0.1    & 0.962    \\ \hline
            OMDSM+CT$^{*}$ & 0.1 & 0.963    \\ \hline
            {\bf Ours+CT}       & 0.1    & \bf 0.965   \\ \hline
        \end{tabular}
        \caption{Test accuracy comparison on CIFAR-10. Where ``-'' indicates the model does not converge, ``*'' indicates OMDSM obtains best performance with dropout=0.3 while keeping all other CT settings unchanged.}
        \label{tab:cifar}
	\end{minipage}
	\vspace{-20pt}
\end{table}

\textbf{Data Set:}
We use Modelnet10~\cite{maturana2015voxnet} for our comparison. In the data set there are 3D models as well as voxelized versions, which have been augmented by rotating in 12 rotations. We use the provided voxelizations and follow the train/test splits for evaluation. 

\textbf{Networks:}
We use VoxNet~\cite{maturana2015voxnet} for comparison, which is a network architecture to efficiently dealing with large amount of point cloud data by integrating a volumetric occupancy grid representation into a supervised 3D CNN. We use the PyTorch code~\cite{voxnet_github_link}
as our testbed, and tune each approach to report the best accuracy averaged per class. In the default setting, we use SGD with momentum 0.9, training epoch 80, and batch size 128. The initial learning rate is 0.001 and multiplied by 0.31 every 16 epochs.

\textbf{Training Stability \& Accuracy:}
We illustrate the training and testing behavior of each algorithm on ModelNet10 in Fig.~\ref{fig:ModelNet10_voxnet_4_cases}.  
As we can see, different to the case of PointNet, there is no significant difference in both training and testing for VoxNet. To further verify the performance, we summarize the best test accuracy in Table \ref{tab:ModelNet10_voxnet_4_cases_table}, where our improvement is marginal. Interestingly, we find similar observations in image classification that all the OR algorithms perform equally well, and no obvious advantage over conventional training with careful tuning. For instance, we list our classification results of baselines, OrthoReg~\cite{rodriguez2016regularizing}, and OMSDM~\cite{huang2017orthogonal} on CIFAR-10~\cite{krizhevsky2014cifar} based on Wide ResNet~\cite{zagoruyko2016wide} with width of 28, depth of 10 in Table~\ref{tab:cifar}.

In summary, we observe that OR is more useful for MLP than for CNNs to improve the training stability. We believe one of the key reasons is the filters in CNNs are much better structured due to the input data such as images and 3D volumes, so that learning orthogonal filters becomes unnecessary. In contrast, the input data for MLP is much less structured individually where orthogonal filters can better cover feature space. In both cases, however, our self-regularization outperforms the state-of-the-art OR algorithms.

\textbf{Angular Distributions:}
We also illustrate the angular distributions of learned models on MondelNet10 in Fig.~\ref{fig:ModelNet10_distributions}. Again all the distributions form Gaussian-like shapes with mean close to $\ang{90}$ and relatively small variance. Together with Fig.~\ref{fig:learned_angular}, we conclude that angular distributions may not be a good indicator for selecting good deep models, but their statistics are good for self-regularization.

\section{Conclusion}
In this paper, we manage to identify that the real gain of orthogonality regularization (OR) in deep learning is to better stabilize the training, leading to faster convergence and better generalization. In terms of accuracy, existing OR algorithms perform no better than the conventional training algorithm with weight decay, dropout, and batch normalization, statistically speaking. Instead, we propose a self-regularization method as an architectural plug-in that can be easily integrated with an arbitrary network to learn orthogonal filters. We utilize LSH to compute the filter angles approximately based on the filter responses, and push the mean and variance of such angles towards $\ang{90}$ and $\ang{0}$, respectively. Empirical results on point cloud classification demonstrate the superiority of self-regularization.

{\small
\bibliographystyle{ieee}
\bibliography{egbib}
}
\balance

\end{document}